\documentclass{article}

\usepackage[section]{placeins}
\usepackage{microtype}
\usepackage{graphicx}
\usepackage{booktabs} 
\usepackage{times}
\usepackage{soul}
\usepackage{tabularx}
\usepackage{url}
\usepackage{tipa}
\usepackage[hidelinks]{hyperref}
\usepackage[utf8]{inputenc}
\usepackage[small,skip=10pt]{caption}
\usepackage{multirow}
\usepackage{bbding}
\usepackage{diagbox}
\usepackage{graphicx}
\usepackage{adjustbox}
\usepackage{booktabs}
\usepackage{float}
\usepackage{subfig}
\usepackage{color}
\usepackage{mathrsfs}
\usepackage{stfloats}
\usepackage{amssymb}
\usepackage[square, numbers, sort&compress,]{natbib}

\usepackage[accepted]{icml2024}

\usepackage{hyperref}

\usepackage[accepted]{icml2024}

\usepackage{amsmath}
\usepackage{amssymb}
\usepackage{mathtools}
\usepackage{amsthm}

\usepackage[capitalize,noabbrev]{cleveref}

\theoremstyle{plain}
\newtheorem{theorem}{Theorem}[section]

\newtheorem{lemma}[theorem]{Lemma}

\theoremstyle{definition}

\theoremstyle{remark}

\usepackage[textsize=tiny]{todonotes}

\icmltitlerunning{Exploring the Low-Pass Filtering Behavior in Image Super-Resolution}

\begin{document}

\twocolumn[
\icmltitle{Exploring the Low-Pass Filtering Behavior in Image Super-Resolution}




\begin{icmlauthorlist}
\icmlauthor{Haoyu Deng}{uestc}
\icmlauthor{Zijing Xu}{uestc}
\icmlauthor{Yule Duan}{uestc}
\icmlauthor{Xiao Wu}{uestc}
\icmlauthor{Wenjie Shu}{uestc}
\icmlauthor{Liang-Jian Deng}{uestc}

\end{icmlauthorlist}

\icmlaffiliation{uestc}{University of Electronic Science and Technology of China}

\icmlcorrespondingauthor{Liang-Jian Deng}{liangjian.deng@uestc.edu.cn}

\icmlkeywords{Machine Learning, ICML}

\vskip 0.3in
]



\printAffiliationsAndNotice{} 

\Crefname{section}{Sec.}{Sec.}
\Crefname{table}{Tab.}{Tab.}
\Crefname{figure}{Fig.}{Fig.}
\Crefname{lemma}{Lemma}{Lemma}
\Crefname{theorem}{Lemma}{Lemma}
\Crefname{equation}{Eq.}{Eq.}
\Crefname{proof}{Proof:0-=}{Proof:}
\Crefname{appendix}{Appx.}{Appx.}
\begin{abstract}
Deep neural networks for image super-resolution (ISR) have shown significant advantages over traditional approaches like the interpolation. However, they are often criticized as `black boxes' compared to traditional approaches with solid mathematical foundations. In this paper, we attempt to interpret the behavior of deep neural networks in ISR using theories from the field of signal processing. First, we report an intriguing phenomenon, referred to as `the sinc phenomenon.' It occurs when an impulse input is fed to a neural network. Then, building on this observation, we propose a method named Hybrid Response Analysis (HyRA) to analyze the behavior of neural networks in ISR tasks. Specifically, HyRA decomposes a neural network into a parallel connection of a linear system and a non-linear system and demonstrates that the linear system functions as a low-pass filter while the non-linear system injects high-frequency information. Finally, to quantify the injected high-frequency information, we introduce a metric for image-to-image tasks called Frequency Spectrum Distribution Similarity (FSDS). FSDS reflects the distribution similarity of different frequency components and can capture nuances that traditional metrics may overlook. Code, videos and raw experimental results for this paper can be found in: \href{https://github.com/RisingEntropy/LPFInISR}{https://github.com/RisingEntropy/LPFInISR}.
\end{abstract}
Please refer to \Cref{tab:symbols} for notation conventions.
\section{Introduction}
\begin{figure}[h]
    \centering
    \includegraphics[scale=0.13]{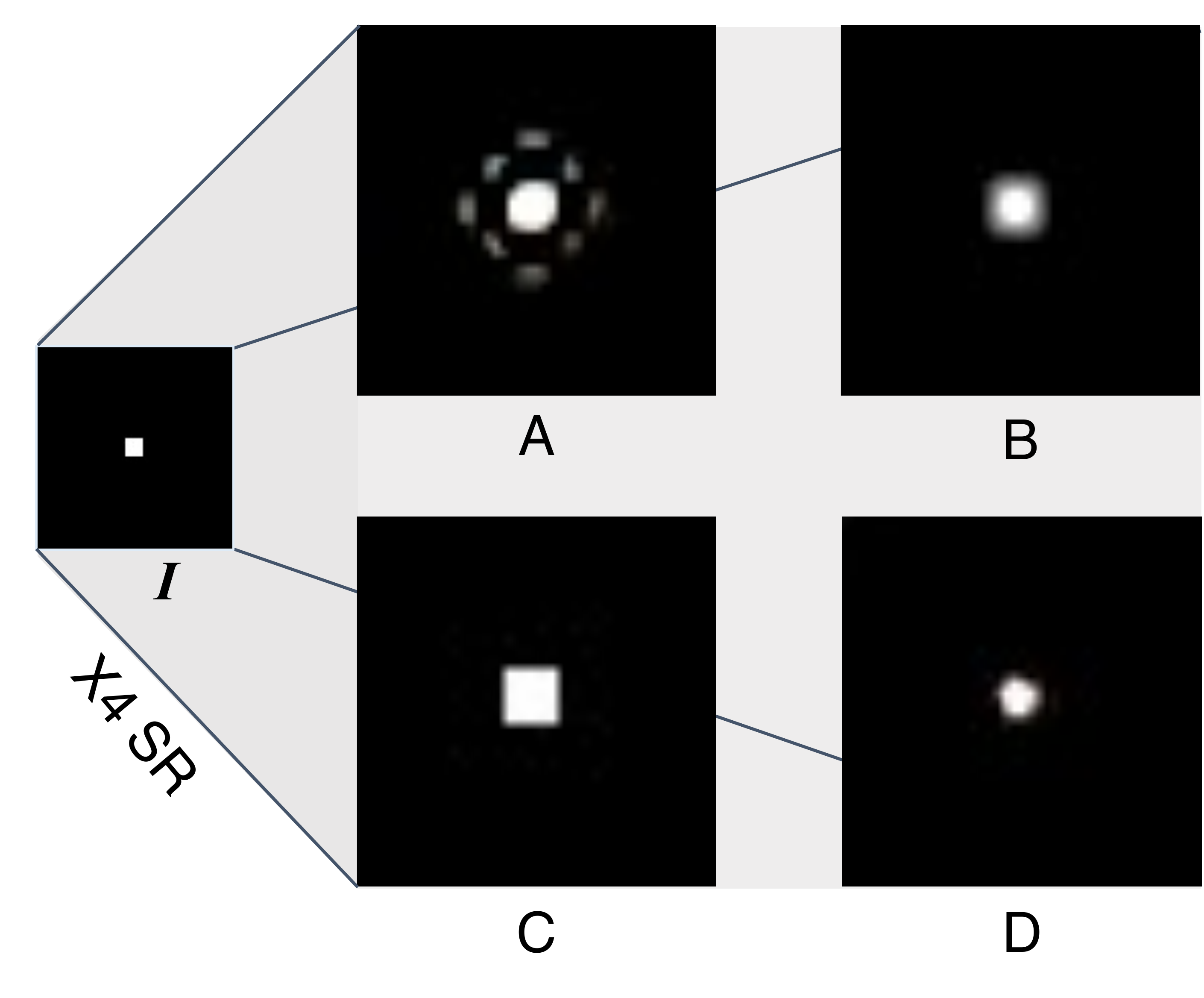}
    \caption{$I$ is an image in which only the central pixel is 1 and the other pixels are 0. What would the result look like if image I is super-resolved using a neural network, A, B, C, or D? Surprisingly, the answer is A. We name this phenomenon as \textbf{the sinc phenomenon}. In this paper, we give a possible explanation for this phenomenon.}
    \label{fig:title_figure}
\end{figure}
The goal of image super-resolution (ISR) is to reconstruct low-resolution (LR) images into high-resolution (HR) images through various techniques. In recent years, with advances in deep learning, growing ISR methods using neural networks are proposed, bringing the development of ISR into a new level. While impressive results persistently arise, the mechanism under ISR networks remain largely unexplored, leading to criticism that they are considered black boxes. In comparison, traditional methods, such as interpolation or filtering, have strong interpretability. Despite the principles of traditional methods and neural networks are different, we can still attempt to explain the behavior of ISR networks using theories from traditional methods. In this paper, following this line of thought, we successfully utilize theories from the field of signal processing techniques to explain the performance of neural networks in the ISR task.
\par
The target of the ISR task is to upsample a two-dimensional signal. In traditional signal processing theory \citep{dsp_book, signalandsystem}, a feasible method for upsample involves restoring a discrete low-sampling-rate signal to a continuous signal using a low-pass filter, and then sampling the continuous signal at a higher rate to obtain a high-sampling-rate signal. An intriguing aspect of this process is that when we try to upsample a Dirac $\delta$ signal, we will finally get a sinc signal since the sinc signal is the time-domain waveform of a low-pass filter, (for details about this, please refer to  \Cref{sec:preliminaries} ). Given this, we can conjecture: if neural networks exhibit similar behavior, then when attempt to super-resolve a Dirac $\delta$ signal, the resultant outcome would also be a sinc signal. As shown in \Cref{fig:title_figure}, we indeed observe this phenomenon, and we name it as `\textit{the sinc phenomenon}'. This phenomenon establishes a connection between traditional signal processing theory and the interpretability of neural networks, thus helping us form a deeper understanding of ISR networks.
\par
Building upon the sinc phenomenon, we further propose a method named HyRA\footnote{Pronounce as [\textipa{haI'rA:}]}, which stands for Hybrid Response Analysis. HyRA considers the neural network as a parallel combination of a linear system and a non-linear system with a zero impulse response. It further indicates that this linear system functions as a low-pass filter, while the non-linear system utilizes the learned prior knowledge to inject high-frequency information. By employing HyRA, we can analyze performance bottlenecks in neural networks, discerning whether the issue lies in inadequate preservation of low-frequency components or insufficient injection of high-frequency components. This analysis facilitates the proposal of targeted improvements for enhanced adaptability.
\par
Given that the non-linear component is injecting high-frequency information, there is a pressing need for a metric to quantitatively describe the extent of the injected high frequencies. Previous metrics, like PSNR, SSIM \citep{SSIM} and LPIPS \citep{LPIPS}, have not approached the evaluation of images from a frequency perspective. Therefore, we propose the frequency spectrum distribution similarity (FSDS), a metric that evaluates image quality based on the power distribution in the frequency spectrum.
\par
In summary, our contribution can be concluded as:
\begin{itemize}
    \item We report an intriguing phenomenon: the impulse responses of image super-resolution (ISR) networks are sinc functions, representing the temporal waveform of a low-pass filter. We name it the 'sinc phenomenon'. This observation helps to establish a connection between signal processing theory and neural networks. Moreover, we find that for a network, the more similar the impulse response is to the sinc function, the better performance it produces.
    \item In order to further explain the performance of neural networks in the ISR task through this phenomenon, we introduce HyRA. HyRA considers the neural network as a parallel combination of a linear system and a non-linear system with a zero impulse response. It points out that the linear system operates as a low-pass filter, while the non-linear system injects high-frequency information.

    \item To quantitatively describe the injection of high frequencies, we introduce the FSDS metric. FSDS measures image quality using frequency spectrum produced by FFT and can reflect high-frequency distortions that previous metrics fail to capture.
\end{itemize}
\section{Related works}
\subsection{Super Resolution Using Neural Networks}
Recent review articles in ISR include fixed-scale super-resolution  \citep{overview_fix} and arbitrary-scale super-resolution review  \citep{overview_arbsr}. There are various architectures of mainstream ISR backbone networks, including CNN-style backbones  \citep{edsr,rdn,IMDN,RCAN,CARN}, transformer-style backbones  \citep{swinir, EQSR} and GAN-style backbone networks  \citep{ESRGAN}, etc. Based on these backbones, researchers have proposed quantitative modules with various functions. For example, ArbSR  \citep{ArbSR} can expand a fixed-scale super-resolution network to an arbitrary-scale ISR network, LTE  \citep{LTE} can enhance local textures, etc. What worth mentioning is that LIIF  \citep{liif} introduces implicit neural representation into ISR for the first time, bringing a new approach for ISR. This paper mainly focuses on approaches that utilize CNN-style or transformer-style backbones. Except for network architectures, numerous datasets have been proposed to facilitate further research. Commonly used datasets for ISR includes Set5  \citep{set5}, Urban100  \citep{Urban100}, Flickr2K  \citep{Flickr2K}, SCI1K  \citep{ITSRN}, DIV2K  \citep{DIV2K}, etc. We evaluate the effectiveness of our proposed FSDS metric on DIV2K dataset. The large size of the DIV2K dataset contributes to increased reliability in our conclusions.

\subsection{Explaining the Behavior of Neural Networks}
Despite neural networks are often criticized as `black boxes,' predecessors have made remarkable efforts to mitigate this situation. Various previous researches have proposed plenty of methods to analyze the behavior of neural networks. 
Since \citeauthor{integrated_gradients}  \citep{integrated_gradients} introduce the integrated gradients (IG) for attribution in classification tasks, numerous researchers have expanded this method to various domains, broadening the scope of attribution beyond classification tasks. Based on IG, \citeauthor{chaodong} \citep{chaodong} propose LAM to analyze the impact of the local patch  on the entire ISR outcome. However, such a method requires manually determined hyper-parameters and baselines, thus introducing subjectivity. Several notable analysis methods utilizing the Fourier transform have been explored in the literature \citep{Training_behavior, xu1, Xu2020, zhang1}. Notably, \citeauthor{Training_behavior} \cite{Training_behavior} propose the Frequency-Principle, claiming its relevance to both convolutional neural networks (CNNs) and fully-connected deep neural networks. According to their proposition, these networks inherently adhere to the Frequency-Principle, wherein training data is systematically acquired in a sequential manner, progressing from low to high frequency. Unlike previous approaches these approaches, HyRA distinguishes itself by employing impulse response to probe the potential mechanisms of deep neural networks in the context of the ISR task.
\section{Preliminaries\label{sec:preliminaries}}

\Cref{sec:system_and_response}-\Cref{sec:freq_overlap} provide a brief overview of signal processing concepts for readers who are not familiar with it. \Cref{sec:system_and_response} introduces the concepts of signals and systems, along with the computation of responses in Linear Time-Invariant (LTI) systems. \Cref{sec:samp_rec} covers the processes of signal sampling and reconstruction. \Cref{sec:freq_overlap} delves into the phenomenon of spectrum aliasing, a factor contributing to the ill-posed nature of the ISR task. And we will introduce applying low-pass filter for ISR here.


 We can employ signal recovery methods to achieve image super-resolution (ISR). Initially, we conceptualize an image as a series of impulse trains in a two-dimensional continuous space, with varying densities representing different resolutions. Then, for the low-resolution image, we begin by implementing low-pass filtering, following the procedure outlined in \Cref{sec:samp_rec}, to obtain the continuous image $I^{\mathrm{cont}}$, This process can be mathematically described as:
\begin{equation}    I^{\mathrm{cont}}_{x,y}=sinc^{\mathrm{\omega}}_{x,y}*I^{\mathrm{LR}}_{x,y},
\end{equation}
where $*$ denotes convolution, $I^{LR}_{x,y}$ is the low resolution image with variant $x, y$ and $I^{\mathrm{cont}}_{x,y}$ is the continuous signal. $sinc^{\omega}_{x,y}$ is a two-dimensional sinc function with parameter $\omega$\footnote{Please refer to \Cref{tab:ft-pair} for Fourier transform pairs}, whose frequency spectrum is an ideal low-pass filter with a passband of $0\sim\omega$. Subsequently, we sample the `conceptually continuous signal' at an elevated sampling rate to acquire a more densely populated two-dimensional sequence of impulse trains, i.e., an image with higher resolution denoted as $I^{SR}$:
\begin{equation}
    I^{SR}_{x, y} = I^{\mathrm{cont}}_{x, y} \cdot s^{\Delta X, \Delta Y}_{x, y}.
\end{equation}
In the equation, $s^{\Delta X, \Delta Y}_{x, y}$ denotes the two-dimensional impulse trains with intervals of $\Delta X$ in $x$ axis and $\Delta Y$ in $y$ axis.
\par
\begin{figure}[h]
     \centering
     \includegraphics[scale=0.12]{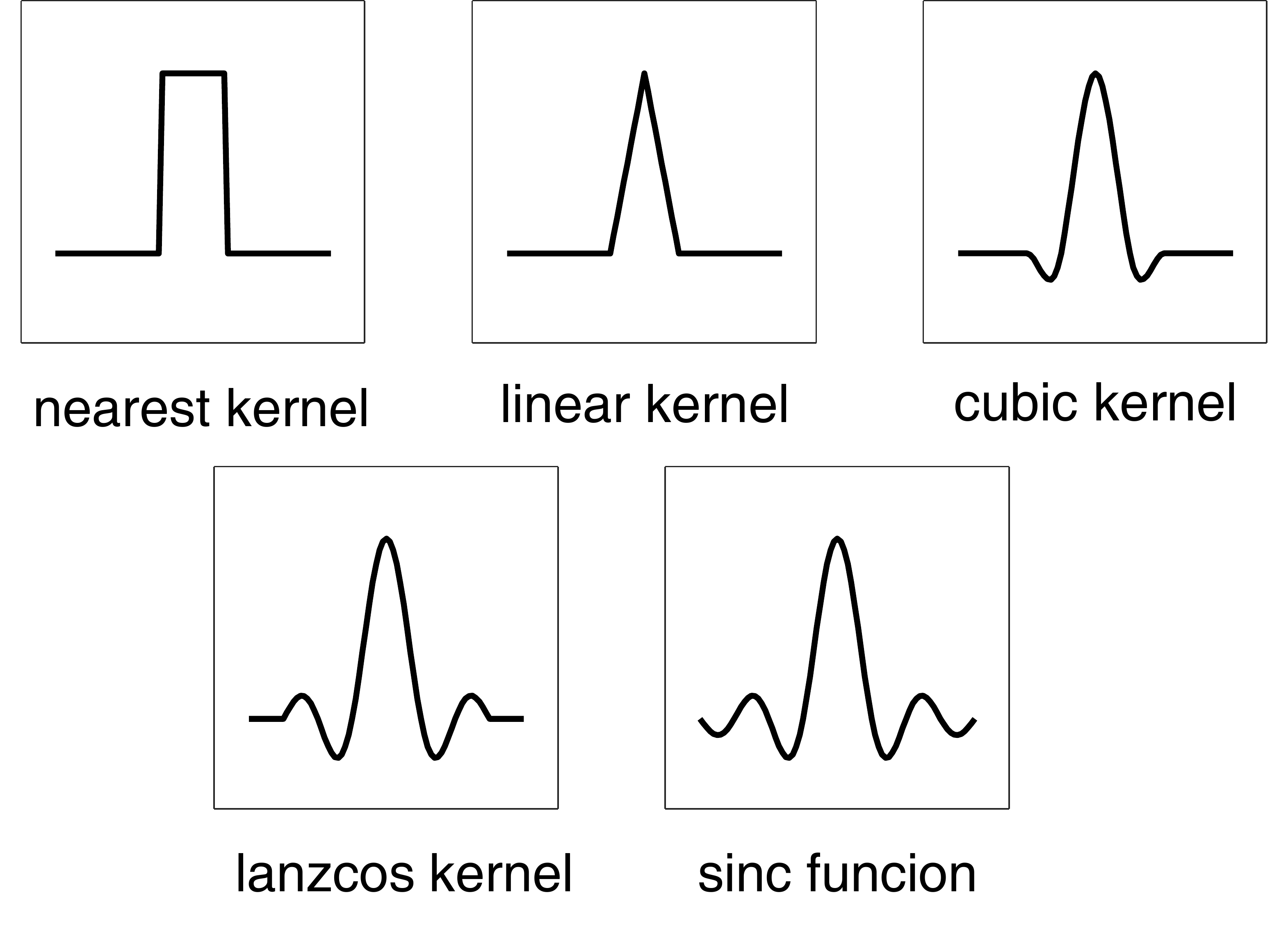}
     \caption{Various interpolation kernels for ISR. They can all be seen as an approximation of sinc function.}
     \label{fig:convolution kernels}
\end{figure}
In fact, commonly used interpolation kernels for ISR, such as nearest-neighbor interpolation, linear interpolation, cubic interpolation, etc., can be seen as approximations of the sinc function considering a balance between computational complexity and effectiveness, as illustrated in \Cref{fig:convolution kernels}. Taking into account the similarity of these interpolation kernels, in this paper, we collectively refer to these parameter-free methods as low-pass filter-based super-resolution methods.
\section{Method}

\subsection{Hybrid Response Analysis (HyRA)\label{sec:HyRA}}
\begin{figure}
     \centering
     \includegraphics[scale=0.1]{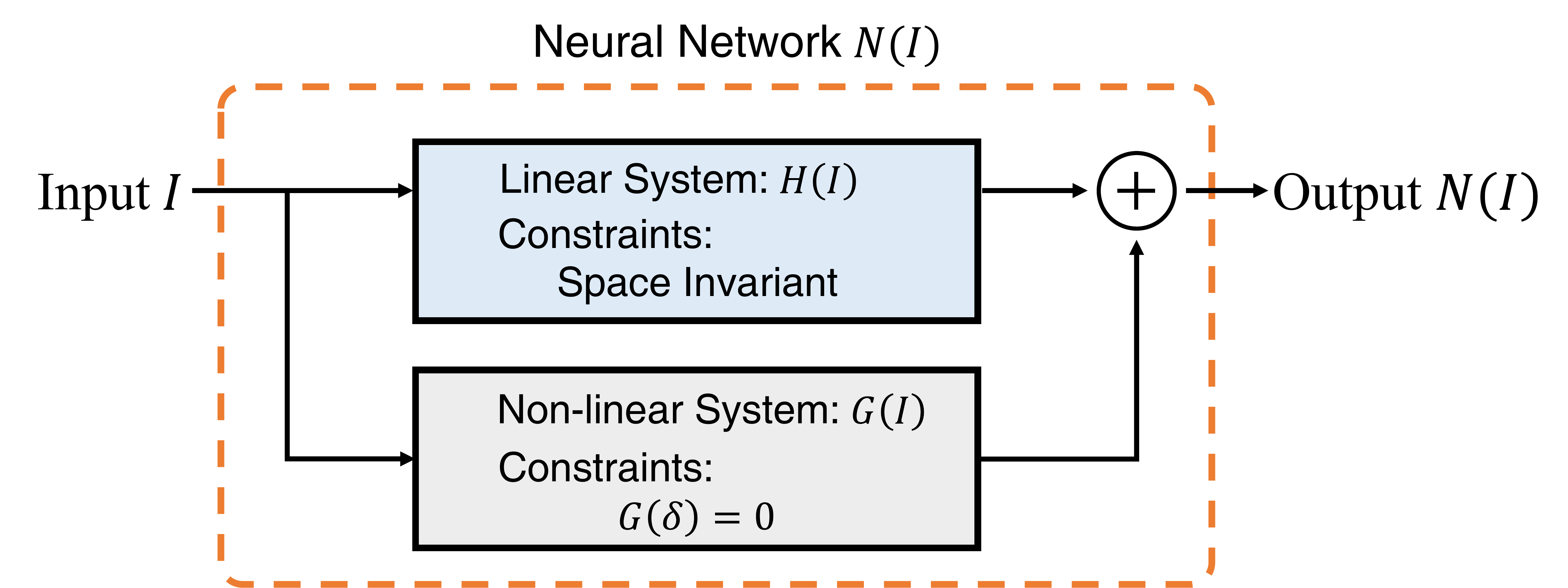}
     \caption{Conceptual diagram of HyRA's core idea.}
     \label{fig:HyRA workflow}
\end{figure}
In this section, we describe the proposed Hybrid Response Analysis (HyRA), which treats the neural network as a combination of a linear system and a non-linear system. Through the impulse response, we can calculate a linear time invariant (LTI) system's output from any input using the convolution operation (see \Cref{sec:system_and_response}). However, since neural networks are nonlinear systems, we cannot apply convolution to analyze them. To further explore the network features, we need to split it into a linear system and a non-linear system, i.e., HyRA. The core concept HyRA is illustrated in \Cref{fig:HyRA workflow}. We denote an ISR network as $N(I)$, where $I$ is the input image. $N(I)$ is a non-linear system that can be expressed as the sum of a linear system and a non-linear system:
\begin{equation}
N(I) = H(I) + G(I)\label{eq:N=HG}.
\end{equation}
In the equation, $H(I)$ represents a linear system, and $G(I)$ represents a non-linear system. Without constraints, such a representation is meaningless because $H(I)$ can be arbitrarily chosen, leading to an infinite variety of representations with the same form but different meanings. To give meaning to this representation, we introduce a constraint: the impulse response of $G(I)$ is zero. With this constraint, both $H(I)$ and $G(I)$ can be uniquely determined. \Cref{lemma:GI=0} demonstrates that under this constraint, $N(I)$ can still be expressed in the form of Eq. \ref{eq:N=HG}. This straightforward method is the essence of HyRA.
\par

For the ISR task, there is a distinctive property known as a `spatially invariant system' \citep{miller1992spatially} associated with it. Consider the definition of time-invariant systems as mentioned in \Cref{sec:system_and_response}, we can naturally extend the concept of in-variance from one-dimensional to two-dimensional space and the definition of spatially invariant systems is: when the input is $I_{x, y}$, the output is $G(I_{x,y}) = O(x, y)$; when the input becomes $I' = I_{x - x_0, y - y_0}$, the output should be $G(I') = O(x - x_0, y - y_0)$. For convolution based architectures, we can easily prove its spatial invariance (see the proof below). For transformer-based architectures, we can still use experiments to prove the spatial invariance (see \Cref{fig:space_invariance}).
\begin{proof}
A convolution operation can be defined as: 
$$ Conv_{i,j}=\sum_{p,q}I_{i-p,j-q}K_{p,q}. $$ 
Then, the shifting operation can be defined as: 
$$ \operatorname{Sh}(i,j)\to(i+k,j+l). $$ 
Combine these two, we then have: 
$$ \begin{aligned} 
Conv_{\operatorname{Sh}(i,j)}&=Conv_{i+k,j+l}\\ &=\sum_{p,q}I_{i+k-p,j+l-q}K_{p,q}\\ &=\operatorname{Sh}(\sum I_{i-p,j-q}K_{p,q}) \\ &=\operatorname{Sh}(Conv_{i,j}) . \end{aligned} $$ 
This is the invariance of a single convolution layer, and still holds for more layers.
\end{proof}


\par
According to HyRA, when we input a Dirac $\delta$ signal to the neural network, we can get the impulse response of the linear system (please recall \Cref{sec:system_and_response}), denoted as $H(\delta)$. For any input $I$, the response of the linear space invariant system can be obtained by convolving the input with the obtained impulse response, which can be expressed as:
\begin{equation}
     H(I)=I*H(\delta),
\end{equation}
where $*$ means the convolution operation. Although the response of the non-linear component cannot be directly computed, if we obtain the final output of the neural network, the non-linear part can be deduced by subtracting the response of the linear component from the final output, namely the non-linear response can be computed as:
\begin{equation}
    \begin{aligned}
        G(I) &= N(I) - H(I)\\
        &= N(I) - I * H(\delta).
    \end{aligned}
\end{equation}


\begin{lemma}
     A neural network $N(I)$ can be expressed as a combination of a linear system $H(I)$ and a non-linear system with an impulse response of zero, i.e., $N(I) = H(I) + G(I)$, where $G(\delta) = 0$. Here, $\delta$ represents the Dirac delta function.
     \label{lemma:GI=0}
\end{lemma}
\begin{proof}$\\$
1) When $G(\delta) = 0$, the conclusion holds.

2) When $G(\delta) \neq 0$, Let $H_1(I) = H(I) + G(\delta) * I$ and $G_1(I) = G(I) - G(\delta) * I$. In this case, $H_1(I)$ remains a linear system and $G'(I)$ remains a non-linear system. The equation $N(I) = H_1(I) + G_1(I)$ holds, and it satisfies $G_1(\delta) = 0$.
\end{proof}
\subsubsection{$H(I)$ is a low-pass filter\label{sec:hi_is_lpfing}}
\begin{figure}[h]
    \centering
    \includegraphics[scale=0.1]{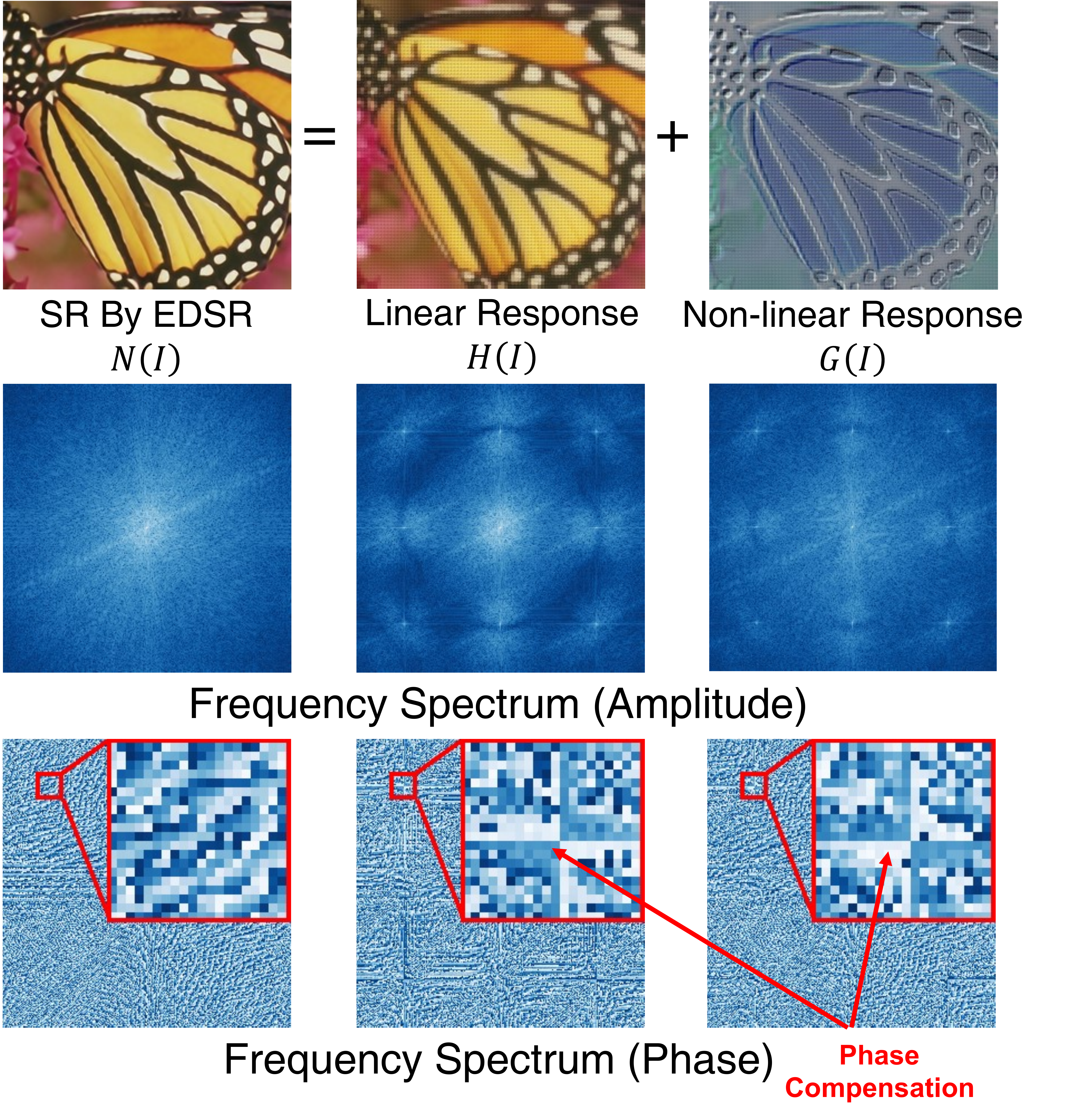}
    \caption{Top row: a super-resolved image by \cite{edsr} can be viewed as the summation of a linear response obtained by convolving impulse response with the input and the non-linear response gained by subtracting linear-part from the ISR result. Second row: the corresponding frequency spectrum amplitude of the top row. Third row: the corresponding frequency spectrum phase of the top row. The phase compensation indicates that the non-linear part is compensating distortion.}
    \label{fig:lin_nonline_resp}
\end{figure}
In \Cref{sec:preliminaries}, we mention that a simple low-pass filter achieves ISR functionality. Do neural networks possess low-pass filters internally? If this hypothesis is valid, according to the principle of HyRA, when we input a Dirac $\delta$ signal into the neural network $N(I)$, the output should be the impulse response of the low-pass filter, i.e., the sinc function (please recall \Cref{sec:system_and_response} and \Cref{tab:ft-pair}).
In the experiment section (\Cref{sec:sinc figure}), we conduct tests on three mainstream ISR backbones and some derived methods. We find that their impulse responses are sinc functions\footnote{Strictly speaking, it is a windowed sinc function. Regarding the windowing operation, please refer to \Cref{sec:app_windowing}}. Now, with both the impulse response and spatial invariance property, we can compute the response of the linear system $H(I)$ to any input through convolution: 
\begin{equation}
        \begin{aligned}
            H(I)_{x,y} &= I_{x,y} * H(\delta)\\ 
            &=\iint_{(\tau,u)\in\mathbb{R}^2} I_{\tau,u}H(\delta)_{x-\tau,y-u}\mathrm{d}\tau\mathrm{d}u.
        \end{aligned}
\end{equation}

In a practical scenario, when dealing with a two-dimensional impulse array represented by $I$, the integration process can be effectively substituted with summation, incorporating appropriate padding. Despite the convolution operator in PyTorch \citep{Paszke2019PyTorchAI} being inherently a correlation operator, the symmetric nature of the sinc function allows for its seamless utilization within such an operator. We present a toy example in \Cref{fig:lin_nonline_resp} in which we compute the response of the linear component of the EDSR network \citep{edsr} during ISR. Observing the experimental results, we notice that the linear function $H(I)$ essentially achieves super-resolution, but there are some issues: edge blurring and the presence of grid-like distortions. 
 
 The edge is blurred because the low-pass filter removes some high-frequency details. In the frequency spectrum, it is manifested as a relatively small range of diffusion of the central bright spot towards the surroundings. This implies that the image has more low-frequency components and fewer high-frequency components. Such an outcome is the inevitable consequence of applying the low-pass filter.
 
 When computing the response of the linear system, we first perform zero-interpolation on the low-resolution image to achieve the target spatial size. This operation leads to periodic extension in the frequency spectrum\footnote{Please refer to \Cref{sec:freq_spec_reduplicate} in the Appendix for details about the periodic extension.}. Since this low-pass filter is not a complete ideal filter, but an ideal filter truncated by a certain window function, its filtering performance is weakened by the window function. The weakened filter cannot completely eliminate the extended spectrum, meaning the attenuation in the stopband is insufficient, as referred to in signal processing, thus causing such gird-like distortions.
 
In summary, the linear system $H(I)$ (the low-pass filter approximated by the neural network) can achieve super-resolution functionality, but it is not perfect. On one hand, the low-pass filter determines that the image is blurred, lacking high frequencies. On the other hand, the filter is windowed, leading to a weakened filtering performance and resulting in grid-like distortions. These issues will be compensated for by the nonlinear system $G(I)$.

\subsubsection{$G(I)$ Injects high-frequency information\label{sec:GIinject}}
\begin{figure}
    \centering
    \includegraphics[scale=0.095]{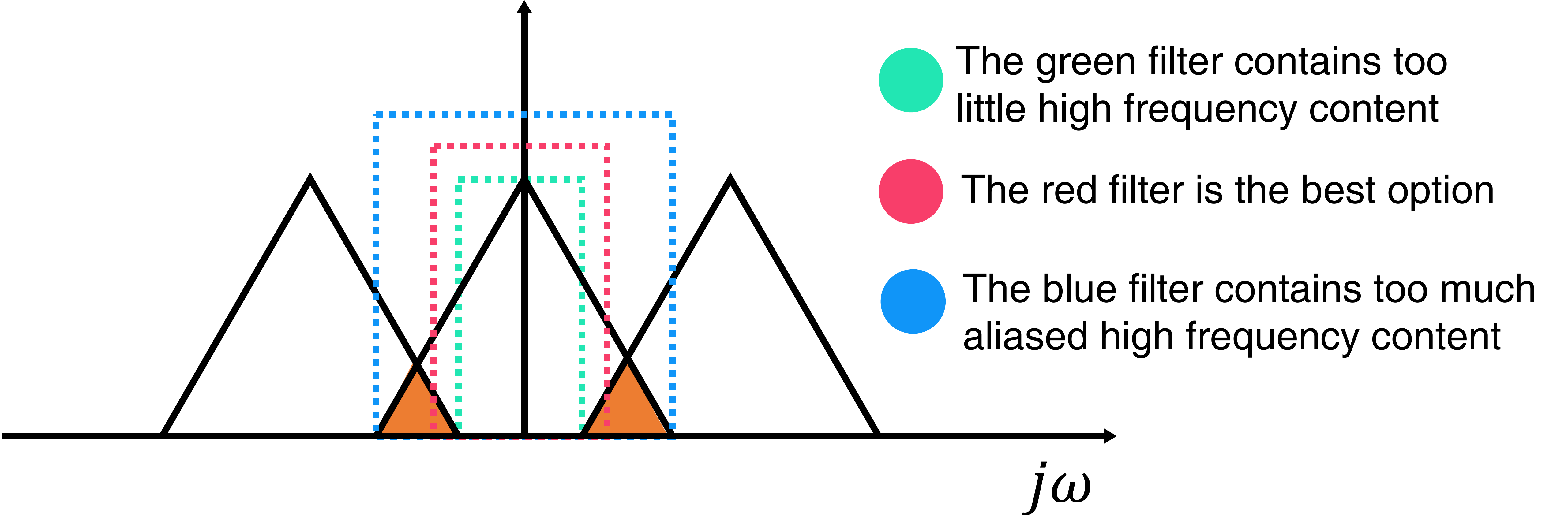}
    \caption{An illustration of how the passband width of a low-pass filter affects its ISR results. A too wide passband or a too narrow passband can result in a decline in performance.}
    \label{fig:filters}
\end{figure}

%

Though a low-pass filter can achieve ISR (please refer to \Cref{sec:preliminaries}), its performance can never surpass a well-trained neural network. The outcome of a low-pass filter varies with respect to the passband width, as depicted in \Cref{fig:filters}. However, information outside the passband will be completely wiped out, causing an observable detail loss in high-frequency components. On the contrary, the non-linear part of neural networks is able to inject information in high-frequency domain based on learned or structural priors. Moreover, it can compensate the grid-like distortions brought by the windowed low-pass filter. Together with the linear part, neural networks function as the superset of low-pass filter, retaining both high and low frequency information. 
\par
We compute the non-linear response and its frequency spectrum of the neural network using the proposed HyRA paradigm. In the toy example presented in \Cref{fig:lin_nonline_resp}, it can be noticed that the response of the non-linear component exhibits sharper edges. Compared with the frequency spectrum of the ISR results, the central bright spot in the response of $G(I)$ spreads to a larger range, indicating that more power is distributed into the high-frequency domain. Almost all the components of the high-frequency part in the final ISR result are contributed by the non-linear component.
\par
As mentioned in \Cref{sec:hi_is_lpfing}, the non-linear component also plays a crucial role in compensating for the distortion introduced by $H(I)$. Examining the response of $G(I)$, we note that it also exhibits grid-like distortions, matching those in the response of $H(I)$. This allows for the cancellation of the grid-like distortions, achieving the final goal of ISR. As shown in \Cref{fig:lin_nonline_resp}, upon observing the frequency spectrum, bright spots corresponding to the amplitude spectrum of $H(I)$ exist in all four corners of the amplitude spectrum of $G(I)$. However, the phase spectrum of $G(I)$ is in compensation of the phase spectrum of $H(I)$, indicating that the grid-like distortion is `erased' here.

In summary, the non-linear component $G(I)$ serves to inject high-frequency details learned during training to compensate for the loss of high frequencies introduced by the low-pass filter. Simultaneously, it addresses distortions arising from the imperfect performance of the low-pass filter.

\subsection{Frequency Spectrum Distribution Similarity (FSDS)\label{sec:fsds}}
In this section, we introduce the FSDS metric to quantitatively describe the so called the `injected high frequencies' as discussed in \Cref{sec:GIinject}.

\subsubsection{Motivation and Method}
\begin{figure}[h]
    \centering
    \includegraphics[scale=0.15]{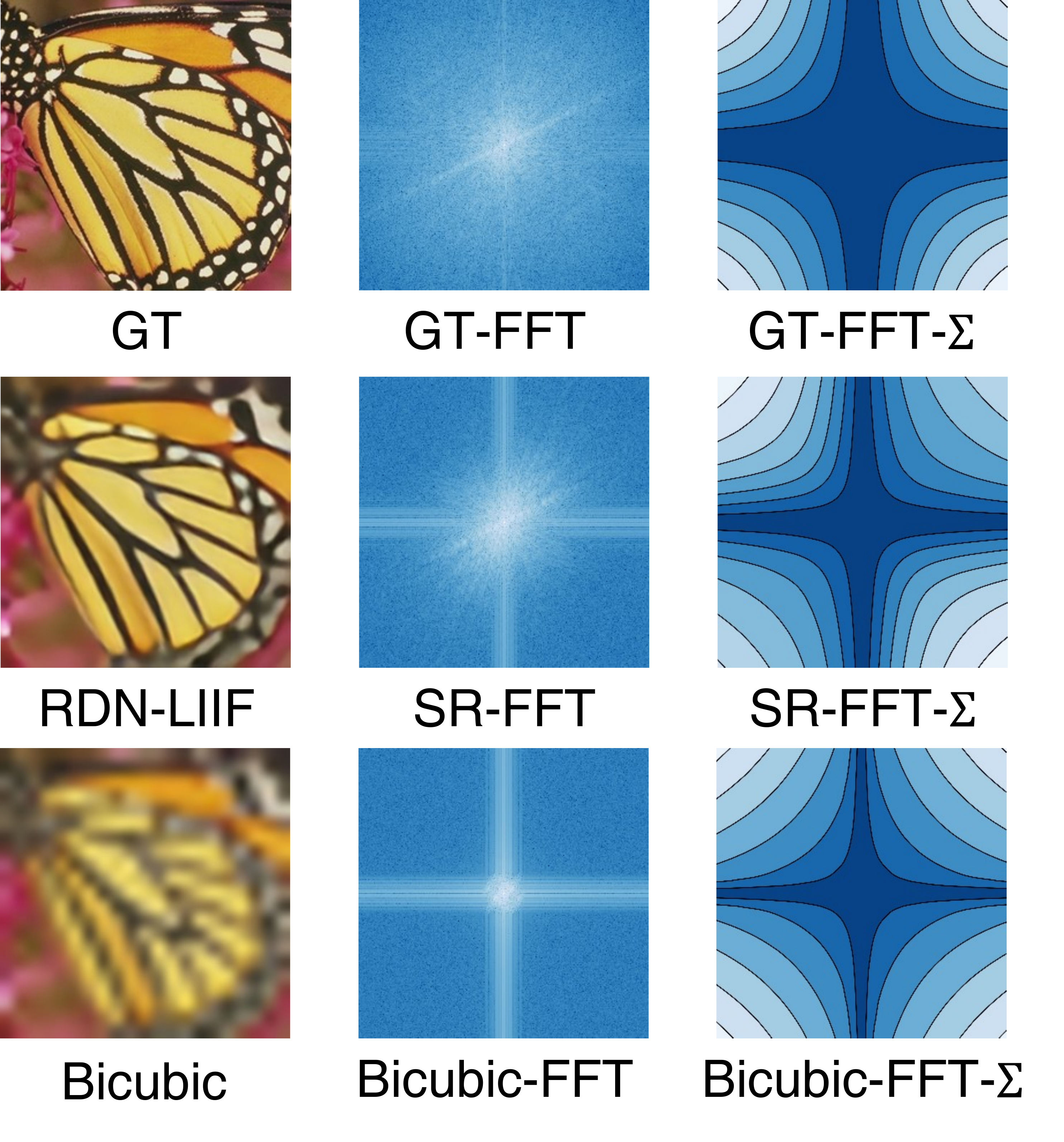}
    \caption{X-FFT-$\Sigma$ denotes the integrated frequency spectrum, the integration path is from origin to infinty in every quadrant. Columns 1 and 2 in the figure respectively show that the differences in the results of different ISR methods can be reflected in the frequency spectrum. Column 3 presents the integral of the spectrum from low to high frequencies in a contour plot. The distribution of contours visually represents the distinct distribution of different frequency components.
    }
    \label{fig:fsds-visualization}
\end{figure}
Since we need to measure the components of injected high frequencies, we must delve into the issue from a frequency spectrum perspective. However, commonly used metrics such as PSNR, SSIM \citep{SSIM}, and LPIPS \citep{LPIPS} do not measure the quality of an image from a spectral perspective.

Additionally, we've noted that the frequency domain distribution in the ISR field can significantly impact downstream applications \citep{Yu2023, Xu_2020_CVPR}. Consequently, we propose that evaluating the ISR effectiveness of a network requires a thorough assessment of its performance in the frequency spectrum. This involves examining the similarity in frequency spectrum between the low-resolution image and the high-resolution image. The Frequency Spectrum Distribution Similarity (FSDS) metric integrates the power distribution maps of the spectrum for both images. The difference is then calculated to generate an error map, and the total sum of its absolute values is computed.
\par
For an image $I^{\mathrm{HR}}_{x,y}$, to minimize the impact of the data input range on the results, we normalize the input data and then perform a two-dimensional Fourier transform to obtain $I^\mathrm{HR}_{j\omega_1, j\omega_2}$, which can be mathematically described as:
\begin{equation}
    I^\mathrm{HR}_{j\omega_1, j\omega_2}=\mathcal{F}\left[\frac{I^\mathrm{HR}-E(I^\mathrm{HR})}{\sigma(I^\mathrm{HR})}\right],
\end{equation}
where $E(I^\mathrm{HR})$ and $\sigma(I^\mathrm{HR})$ are the mean value and variance of $I^\mathrm{HR}$ respectively. Similarly, we perform a Fourier transform on the ISR image to obtain $I^\mathrm{SR}_{j\omega_1,j\omega_2}$. It is worth noting that unlike other metrics, such as PSNR and SSIM\cite{SSIM}, which do not incorporate normalization, FSDS is specifically designed to accentuate numerical variations due to its emphasis on numerical changes rather than absolute numerical values. Then, the complex integration of the two spectrum is performed, providing the power distribution map $D^{\mathrm{HR}}$, which is defined as:
\begin{equation}
    D^{\mathrm{HR}} = \iint_{(\omega_1, \omega_2)\in \mathbb{R}^2}I^\mathrm{HR}\mathrm{d}\omega_1\mathrm{d}\omega_2.
\end{equation}
Similarly, we can obtain $D^{SR}$. Subsequently, the difference between $D^{HR}$ and $D^{SR}$ is calculated, providing a difference map $D^{\mathrm{diff}}$ of their power distribution:
\begin{equation}
    D^{\mathrm{diff}}=D^\mathrm{HR}-D^\mathrm{SR}.
\end{equation}
Finally, we define the frequency spectrum distribution similarity (FSDS) as:
\begin{equation}
    \operatorname{FSDS}=-10\log_{10}\frac{\iint_{(\omega_1, \omega_2)\in \mathbb{R}^2}|D^\mathrm{diff}|^2\mathrm{d}\omega_1\mathrm{d}\omega_2}{\iint_{(\omega_1, \omega_2)\in \mathbb{R}^2}|D^\mathrm{HR}|^2\mathrm{d}\omega_1\mathrm{d}\omega_2} ,
    \label{eq:fsds}
\end{equation}
where $|\cdot|$ represents taking the magnitude of a complex number. Considering a more concise description of a larger dynamic range, logarithm is taken. A larger FSDS value indicates that the two images are closer, thereby suggesting better ISR results.

\subsubsection{The merits of FSDS}

\begin{figure*}
    \centering
    \includegraphics[scale=0.12]{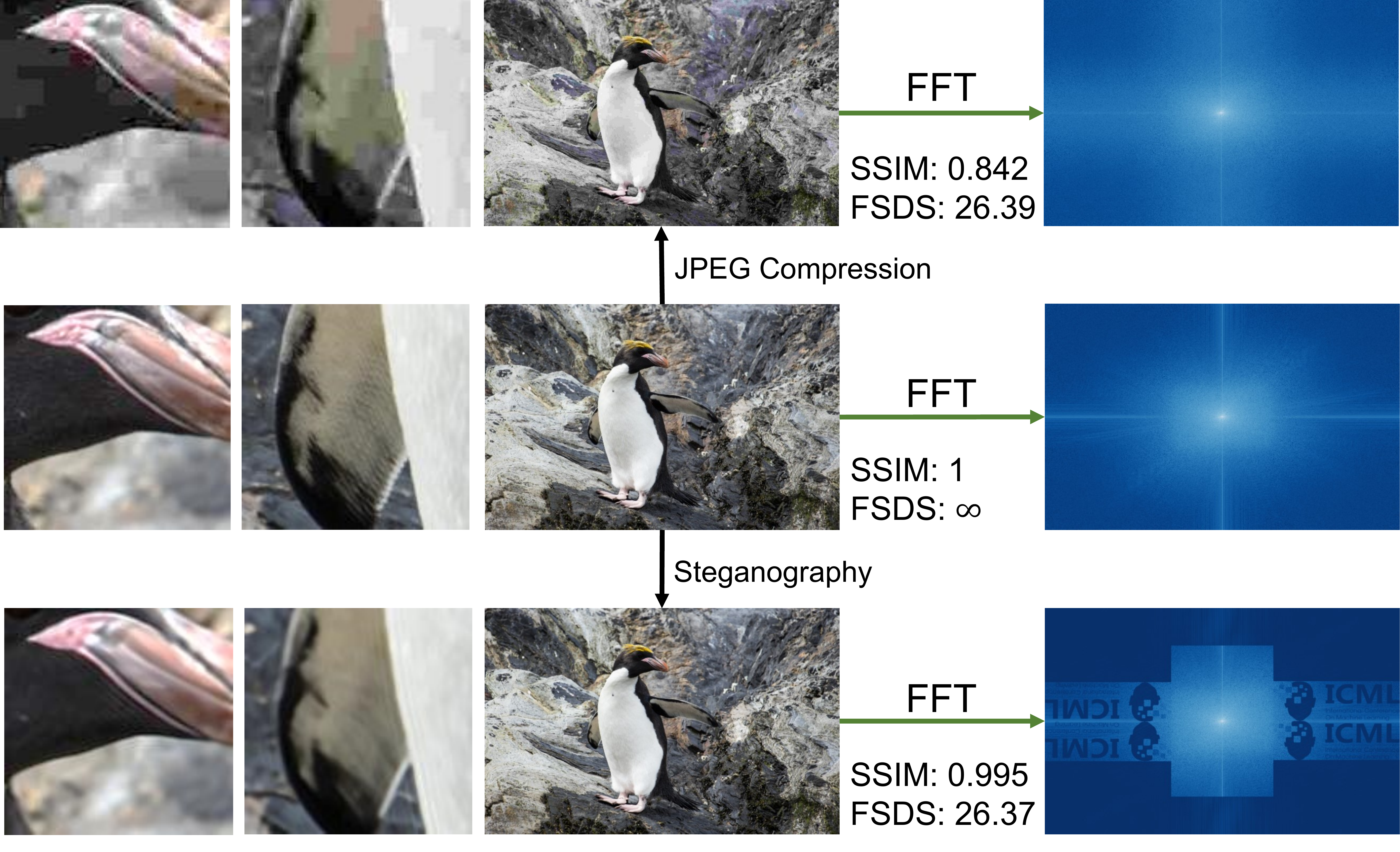}
    \caption{A comparison of SSIM and FSDS in JPEG compression and steganography. As can be seen, SSIM fails to reflect distortion brought by steganography, while FSDS captures both cases of distortion.}
    \label{fig:ssim_bug}
\end{figure*}

Previous image evaluation metrics, such as PSNR, SSIM \citep{SSIM}, have focused on statistical or structural features of images, but no work has evaluated images from the perspective of their frequency spectrum. {The spectrum is the concentrated expression of components with different changing rates in a signal or image. It is crucial for capturing details, eliminating noise, and comprehensively understanding image features. In image processing, spectrum analysis provides a more accurate evaluation, particularly playing a key role in applications sensitive to details. Due to the nature of Fourier transformation, which involves every pixel of the image in the computation, it encompasses not only information such as signal-to-noise ratio and structural similarity but also the overall similarity of the entire image. 
Therefore, evaluating image quality from the perspective of the spectrum is highly reasonable and necessary. Our FSDS metric can reflect distribution differences by employing a paradigm of integrating first in the frequency spectrum and then comparing. In other words, FSDS not only reflects the signal-to-noise ratio captured by the PSNR metric and the structural similarity indicated by the SSIM metric, but also captures features that these two metrics cannot represent. 
In the next paragraph, we will use two toy examples to demonstrate the rationale and advantages of FSDS.

From \Cref{fig:fsds-visualization}, it can be observed that images obtained by different ISR methods have different proportions of high-frequency components (the center of the spectrum figure represents low frequency, while higher frequencies extend outward). After integration, this is reflected in the varying widths of the dark cross-shaped patterns in the center. A narrower width indicates a higher proportion of low-frequency components in the spectrum, and vice versa. Existing methods may not effectively capture the loss of high-frequency components with low power in the frequency spectrum. Performing information steganography in the frequency spectrum can effectively highlight this aspect. As shown in \Cref{fig:ssim_bug}, we embed some content in the frequency spectrum of the image. Such steganography causes our FSDS metric to drop to 26.37dB while the SSIM metric remains in a high level of 0.995. 
we can observe that after applying specific steganography to the spectrum of an image, the image exhibits some blurring and oscillation. Such oscillations are actually the Gibbs phenomenon, a typical oscillation phenomenon caused by the loss of high-frequency information.
Meanwhile, when we apply JEPG compression to the image\footnote{In this example, the compression quality is set to 10.}, when FSDS drops to 26.39dB, SSIM together drops to 0.842. \textbf{This toy example demonstrates that there indeed exists some feature SSIM cannot reflect while that can be reflected by FSDS}.

In summary, previous methods may not effectively reflect the situation in the image frequency spectrum, while our proposed FSDS metric can sensitively detect distortions in the frequency spectrum.

\section{Experiments}

\begin{figure}
    \centering
    \includegraphics[scale=0.25]{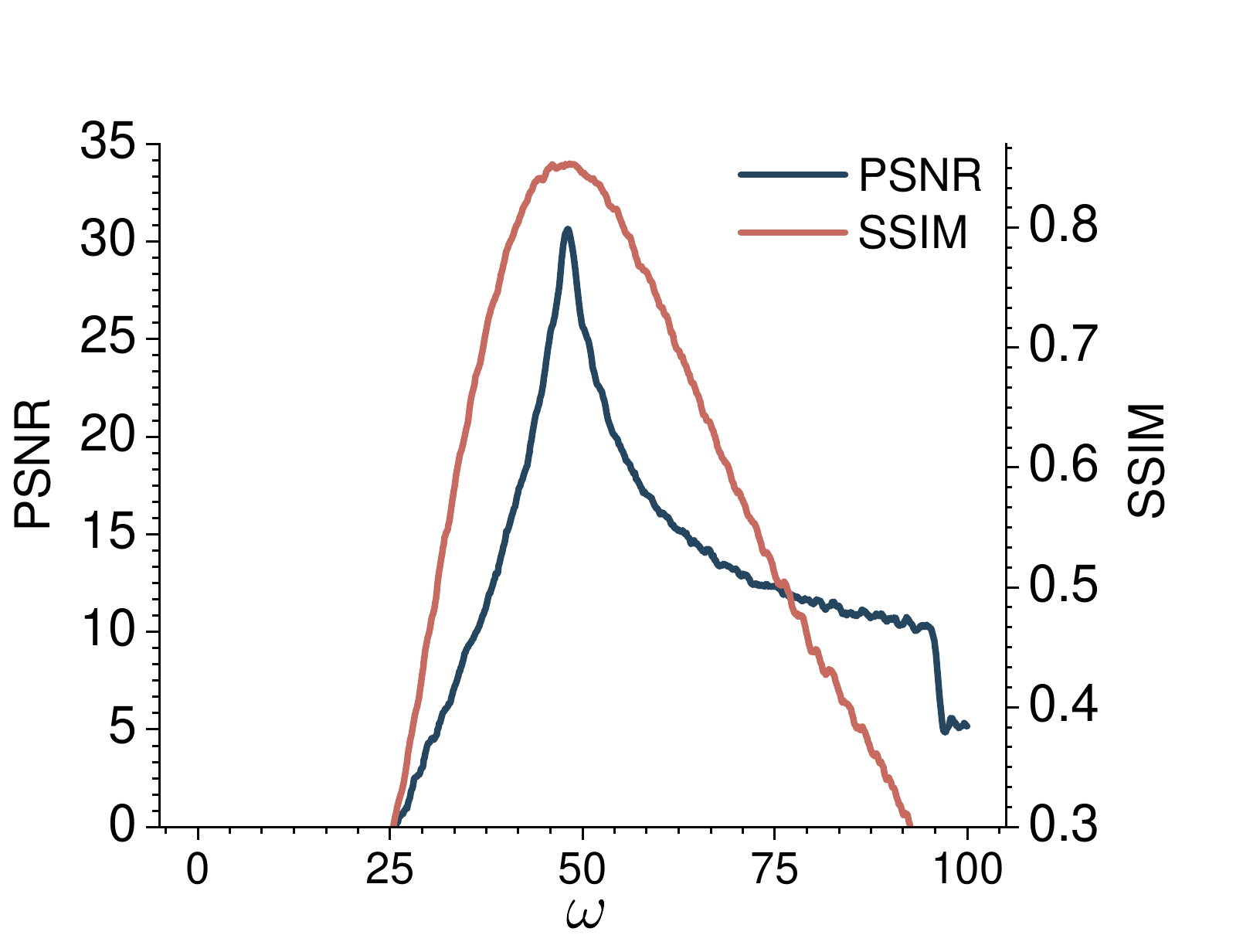}
    \caption{The ISR performance using a low-pass filter shows variations with the cutoff frequency $\omega$. This figure illustrates the results obtained from the $\times$2 ISR task conducted on the DIV2K dataset. To enhance the clarity of the visualization, the curve has been smoothed using a moving average with a window length of 10.}
    \label{fig:lpf-performance}
\end{figure}
\begin{figure*}[h!]
    \centering
    \includegraphics[scale=0.12]{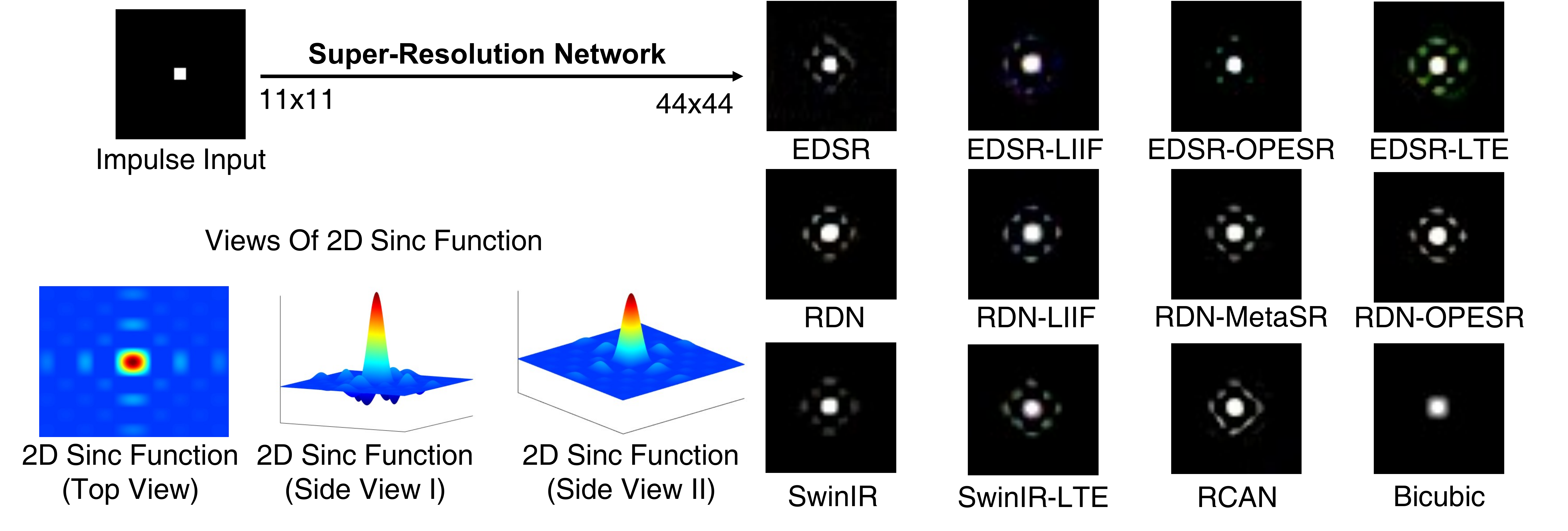}
    \caption{Comparison of impulse responses and the sinc function for several mainstream backbone networks and their derivatives. The impulse response of the bicubic interpolation result is presented as a reference.}
    \label{fig:sinc_impulse_response}
\end{figure*}
Due to the page limitation, we can only present three of the most crucial experiments in this section, namely: 1) the relationship between the low-pass filter passband width and ISR performance; 2) various network impulse responses; 3) a comparison of FSDS metrics with PSNR, SSIM and LPIPS \citep{LPIPS} metrics on the DIV2K dataset. For more experiments, please refer to \Cref{sec:extra_exp}.

\subsection{Experiment on Low-pass Filtering Super-resolution Performance\label{sec:lpf_performance}}

In \Cref{sec:GIinject}, we mention that a vanilla low-pass filter can achieve ISR, we now present an experiment on the relationship between the low-pass filter passband width and ISR performance.
As shown in \Cref{fig:lpf-performance}, we utilized various low-pass filters to perform $\times$2 ISR on the validation set from of DIV2K dataset. Subsequently, we evaluated the ISR results using the PSNR and SSIM metrics.  When $\omega=48$, PSNR reaches its maximum value of $31.40$. When $\omega=45.8$, SSIM reaches its maximum value of $0.87$. We assert that, in terms of neural network performance, for $\times$2 ISR, the PSNR should not fall below $31.40$, and the SSIM should not be lower than $0.87$. Otherwise, it can be considered that the neural network may not effectively capture both low-frequency and high-frequency information.

\subsection{Experiment on Impulse Response\label{sec:sinc figure}}

We select several mainstream backbones and their derivatives commonly used for the ISR task  \citep{edsr, liif, OPE-SR, SRNO, rdn, metasr, swinir, LTE, RCAN} and conduct impulse response tests. The experimental results are compared with the sinc function and depicted in \Cref{fig:sinc_impulse_response}. The input image is an $11\times 11$ image where only the pixel at position $(5,5)$ is white (the values for all three channels at this position are 255, with indices starting from 0), and the rest of the image is black (with values of 0). According to \Cref{tab:ft-pair}, it can be observed that as the ISR factor increases, the central peak of the output sinc function becomes wider and more pronounced. To balance visual saliency and the maximum ISR factor achievable by certain networks, we opted for a 4x ISR factor. Observing the experimental results, we can notice that regardless of the neural network structure used for ISR, whether it's a CNN or a transformer, the impulse response exhibits some degree of similarity to the two-dimensional sinc function. This similarity is particularly pronounced in networks like RDN  \citep{rdn} and RCAN  \citep{RCAN}. Despite some distortion in comparison to the sinc function, EDSR  \citep{edsr}, EQSR  \citep{EQSR}, and their derivatives still exhibit significant features of the sinc function, including the central bright spot and elongated bright patches in the cardinal directions. From \Cref{tab:metrics}, we observe that networks exhibiting superior performance tend to generate impulse responses that closely resemble the sinc function. \textbf{
This observation suggests that preserving low-frequency information more effectively can also enhance performance. However, few previous works has focus on low-frequency, giving us a new idea for furture ISR networks.}

\subsection{Experiment on FSDS Metric\label{sec:fsds_metric}}
\begin{table*}[h!]
    \Large
    \centering
    \renewcommand{\arraystretch}{1.2}
    \resizebox{\textwidth}{!}{
    \begin{tabular}{lcccccccccccccccccccccc}
    \toprule 
         \multirow{3}*{Method}&\multicolumn{5}{c}{{PSNR}}&\multicolumn{5}{c}{{SSIM}}&\multicolumn{5}{c}{{LPIPS}}&\multicolumn{5}{c}{{FSDS (Ours)}}\\
         \cmidrule(l{2pt}r{2pt}){2-6}\cmidrule(l{2pt}r{2pt}){7-11}\cmidrule(l{2pt}r{2pt}){12-16}\cmidrule(l{2pt}r{2pt}){17-21}&{$\times 2$}&{$\times 3$}&{$\times 4$}&{$\times 6$}&{$\times 12$}&{$\times 2$}&{$\times 3$}&{$\times 4$}&{$\times 6$}&{$\times 12$}&{$\times 2$}&{$\times 3$}&{$\times 4$}&{$\times 6$}&{$\times 12$}&{$\times 2$}&{$\times 3$}&{$\times 4$}&{$\times 6$}&{$\times 12$}\\
         \midrule
         EDSR\citep{edsr}&34.63\textcolor{gray}{\textsuperscript{12}}&30.95\textcolor{gray}{\textsuperscript{14}}&28.87\textcolor{gray}{\textsuperscript{16}}&-&-&0.937\textcolor{gray}{\textsuperscript{12}}&0.874\textcolor{gray}{\textsuperscript{13}}&0.816\textcolor{gray}{\textsuperscript{16}}&-&-&0.042\textcolor{gray}{\textsuperscript{12}}&0.101\textcolor{gray}{\textsuperscript{14}}&0.155\textcolor{gray}{\textsuperscript{16}}&-&-&39.21\textcolor{gray}{\textsuperscript{15}}&34.15\textcolor{gray}{\textsuperscript{10}}&31.38\textcolor{gray}{\textsuperscript{7}}&-&-\\
EDSR-LIIF\citep{edsr}&34.55\textcolor{gray}{\textsuperscript{14}}&30.92\textcolor{gray}{\textsuperscript{15}}&28.98\textcolor{gray}{\textsuperscript{15}}&26.76\textcolor{gray}{\textsuperscript{4}}&23.75\textcolor{gray}{\textsuperscript{4}}&0.937\textcolor{gray}{\textsuperscript{15}}&0.874\textcolor{gray}{\textsuperscript{14}}&0.819\textcolor{gray}{\textsuperscript{14}}&0.741\textcolor{gray}{\textsuperscript{4}}&0.633\textcolor{gray}{\textsuperscript{4}}&0.043\textcolor{gray}{\textsuperscript{15}}&0.100\textcolor{gray}{\textsuperscript{13}}&0.153\textcolor{gray}{\textsuperscript{13}}&0.243\textcolor{gray}{\textsuperscript{4}}&\textcolor{blue}{0.428}\textcolor{gray}{\textsuperscript{2}}&39.37\textcolor{gray}{\textsuperscript{13}}&34.53\textcolor{gray}{\textsuperscript{6}}&31.32\textcolor{gray}{\textsuperscript{9}}&28.45\textcolor{gray}{\textsuperscript{4}}&23.15\textcolor{gray}{\textsuperscript{3}}\\
EDSR-OPESR\citep{edsr}&34.34\textcolor{gray}{\textsuperscript{16}}&30.96\textcolor{gray}{\textsuperscript{12}}&29.04\textcolor{gray}{\textsuperscript{12}}&-&-&0.936\textcolor{gray}{\textsuperscript{16}}&0.875\textcolor{gray}{\textsuperscript{11}}&0.820\textcolor{gray}{\textsuperscript{13}}&-&-&0.043\textcolor{gray}{\textsuperscript{13}}&0.100\textcolor{gray}{\textsuperscript{11}}&0.153\textcolor{gray}{\textsuperscript{14}}&-&-&39.80\textcolor{gray}{\textsuperscript{6}}&34.63\textcolor{gray}{\textsuperscript{5}}&31.29\textcolor{gray}{\textsuperscript{11}}&-&-\\
EDSR-SRNO\citep{edsr}&34.72\textcolor{gray}{\textsuperscript{9}}&31.05\textcolor{gray}{\textsuperscript{10}}&29.13\textcolor{gray}{\textsuperscript{10}}&26.90\textcolor{gray}{\textsuperscript{3}}&23.87\textcolor{gray}{\textsuperscript{3}}&0.939\textcolor{gray}{\textsuperscript{9}}&0.876\textcolor{gray}{\textsuperscript{10}}&0.822\textcolor{gray}{\textsuperscript{11}}&0.746\textcolor{gray}{\textsuperscript{3}}&0.638\textcolor{gray}{\textsuperscript{3}}&0.041\textcolor{gray}{\textsuperscript{8}}&0.098\textcolor{gray}{\textsuperscript{10}}&0.149\textcolor{gray}{\textsuperscript{10}}&0.241\textcolor{gray}{\textsuperscript{3}}&0.437\textcolor{gray}{\textsuperscript{4}}&39.53\textcolor{gray}{\textsuperscript{11}}&34.53\textcolor{gray}{\textsuperscript{7}}&31.45\textcolor{gray}{\textsuperscript{6}}&28.46\textcolor{gray}{\textsuperscript{3}}&22.78\textcolor{gray}{\textsuperscript{4}}\\
EDSR-LTE\citep{edsr}&34.61\textcolor{gray}{\textsuperscript{13}}&30.97\textcolor{gray}{\textsuperscript{11}}&29.03\textcolor{gray}{\textsuperscript{14}}&-&-&0.937\textcolor{gray}{\textsuperscript{14}}&0.874\textcolor{gray}{\textsuperscript{12}}&0.820\textcolor{gray}{\textsuperscript{12}}&-&-&0.043\textcolor{gray}{\textsuperscript{14}}&0.100\textcolor{gray}{\textsuperscript{12}}&0.152\textcolor{gray}{\textsuperscript{12}}&-&-&39.29\textcolor{gray}{\textsuperscript{14}}&34.33\textcolor{gray}{\textsuperscript{8}}&31.30\textcolor{gray}{\textsuperscript{10}}&-&-\\
RDN\citep{rdn}&34.69\textcolor{gray}{\textsuperscript{10}}&30.58\textcolor{gray}{\textsuperscript{16}}&29.12\textcolor{gray}{\textsuperscript{11}}&-&-&0.938\textcolor{gray}{\textsuperscript{10}}&0.867\textcolor{gray}{\textsuperscript{16}}&0.823\textcolor{gray}{\textsuperscript{10}}&-&-&0.041\textcolor{gray}{\textsuperscript{9}}&0.106\textcolor{gray}{\textsuperscript{16}}&0.150\textcolor{gray}{\textsuperscript{11}}&-&-&40.02\textcolor{gray}{\textsuperscript{3}}&32.95\textcolor{gray}{\textsuperscript{16}}&31.65\textcolor{gray}{\textsuperscript{4}}&-&-\\
RDN-LIIF\citep{rdn}&34.86\textcolor{gray}{\textsuperscript{8}}&31.21\textcolor{gray}{\textsuperscript{8}}&29.26\textcolor{gray}{\textsuperscript{9}}&\textcolor{blue}{26.99}\textcolor{gray}{\textsuperscript{2}}&\textcolor{blue}{23.93}\textcolor{gray}{\textsuperscript{2}}&0.939\textcolor{gray}{\textsuperscript{8}}&0.879\textcolor{gray}{\textsuperscript{8}}&0.826\textcolor{gray}{\textsuperscript{9}}&\textcolor{blue}{0.749}\textcolor{gray}{\textsuperscript{2}}&\textcolor{blue}{0.639}\textcolor{gray}{\textsuperscript{2}}&0.041\textcolor{gray}{\textsuperscript{10}}&0.096\textcolor{gray}{\textsuperscript{8}}&0.147\textcolor{gray}{\textsuperscript{8}}&\textcolor{red}{0.231}\textcolor{gray}{\textsuperscript{1}}&\textcolor{red}{0.406}\textcolor{gray}{\textsuperscript{1}}&39.69\textcolor{gray}{\textsuperscript{9}}&34.83\textcolor{gray}{\textsuperscript{3}}&31.83\textcolor{gray}{\textsuperscript{3}}&\textcolor{red}{28.89}\textcolor{gray}{\textsuperscript{1}}&\textcolor{red}{23.78}\textcolor{gray}{\textsuperscript{1}}\\
RDN-OPESR\citep{rdn}&34.52\textcolor{gray}{\textsuperscript{15}}&31.19\textcolor{gray}{\textsuperscript{9}}&29.28\textcolor{gray}{\textsuperscript{8}}&-&-&0.938\textcolor{gray}{\textsuperscript{11}}&0.879\textcolor{gray}{\textsuperscript{9}}&0.826\textcolor{gray}{\textsuperscript{8}}&-&-&0.042\textcolor{gray}{\textsuperscript{11}}&0.096\textcolor{gray}{\textsuperscript{7}}&0.148\textcolor{gray}{\textsuperscript{9}}&-&-&\textcolor{blue}{40.19}\textcolor{gray}{\textsuperscript{2}}&\textcolor{blue}{34.96}\textcolor{gray}{\textsuperscript{2}}&31.48\textcolor{gray}{\textsuperscript{5}}&-&-\\
RDN-LTE\citep{rdn}&34.91\textcolor{gray}{\textsuperscript{7}}&31.26\textcolor{gray}{\textsuperscript{7}}&29.31\textcolor{gray}{\textsuperscript{7}}&\textcolor{red}{27.05}\textcolor{gray}{\textsuperscript{1}}&\textcolor{red}{23.99}\textcolor{gray}{\textsuperscript{1}}&0.939\textcolor{gray}{\textsuperscript{7}}&0.879\textcolor{gray}{\textsuperscript{7}}&0.827\textcolor{gray}{\textsuperscript{7}}&\textcolor{red}{0.750}\textcolor{gray}{\textsuperscript{1}}&\textcolor{red}{0.641}\textcolor{gray}{\textsuperscript{1}}&0.041\textcolor{gray}{\textsuperscript{7}}&0.095\textcolor{gray}{\textsuperscript{6}}&0.144\textcolor{gray}{\textsuperscript{6}}&\textcolor{blue}{0.233}\textcolor{gray}{\textsuperscript{2}}&0.431\textcolor{gray}{\textsuperscript{3}}&39.82\textcolor{gray}{\textsuperscript{5}}&34.75\textcolor{gray}{\textsuperscript{4}}&\textcolor{blue}{31.84}\textcolor{gray}{\textsuperscript{2}}&\textcolor{blue}{28.65}\textcolor{gray}{\textsuperscript{2}}&\textcolor{blue}{23.35}\textcolor{gray}{\textsuperscript{2}}\\
SwinIR-classical\citep{swinir}&35.34\textcolor{gray}{\textsuperscript{5}}&31.64\textcolor{gray}{\textsuperscript{5}}&29.63\textcolor{gray}{\textsuperscript{4}}&-&-&0.943\textcolor{gray}{\textsuperscript{5}}&0.885\textcolor{gray}{\textsuperscript{5}}&0.835\textcolor{gray}{\textsuperscript{5}}&-&-&0.038\textcolor{gray}{\textsuperscript{5}}&0.092\textcolor{gray}{\textsuperscript{4}}&0.140\textcolor{gray}{\textsuperscript{5}}&-&-&\textcolor{red}{40.37}\textcolor{gray}{\textsuperscript{1}}&\textcolor{red}{35.13}\textcolor{gray}{\textsuperscript{1}}&\textcolor{red}{32.37}\textcolor{gray}{\textsuperscript{1}}&-&-\\
ITSRN\citep{ITSRN}&32.67\textcolor{gray}{\textsuperscript{17}}&30.49\textcolor{gray}{\textsuperscript{17}}&28.73\textcolor{gray}{\textsuperscript{17}}&26.64\textcolor{gray}{\textsuperscript{5}}&23.72\textcolor{gray}{\textsuperscript{5}}&0.922\textcolor{gray}{\textsuperscript{17}}&0.866\textcolor{gray}{\textsuperscript{17}}&0.813\textcolor{gray}{\textsuperscript{17}}&0.736\textcolor{gray}{\textsuperscript{5}}&0.630\textcolor{gray}{\textsuperscript{5}}&0.052\textcolor{gray}{\textsuperscript{17}}&0.113\textcolor{gray}{\textsuperscript{17}}&0.167\textcolor{gray}{\textsuperscript{17}}&0.271\textcolor{gray}{\textsuperscript{5}}&0.469\textcolor{gray}{\textsuperscript{5}}&31.25\textcolor{gray}{\textsuperscript{18}}&26.18\textcolor{gray}{\textsuperscript{18}}&25.88\textcolor{gray}{\textsuperscript{18}}&25.62\textcolor{gray}{\textsuperscript{5}}&21.57\textcolor{gray}{\textsuperscript{5}}\\

HAT-S\citep{hat}&\textcolor{blue}{35.46}\textcolor{gray}{\textsuperscript{2}}&31.72\textcolor{gray}{\textsuperscript{3}}&29.72\textcolor{gray}{\textsuperscript{3}}&-&-&\textcolor{blue}{0.944}\textcolor{gray}{\textsuperscript{2}}&0.887\textcolor{gray}{\textsuperscript{3}}&0.837\textcolor{gray}{\textsuperscript{3}}&-&-&\textcolor{blue}{0.038}\textcolor{gray}{\textsuperscript{2}}&0.092\textcolor{gray}{\textsuperscript{3}}&0.139\textcolor{gray}{\textsuperscript{3}}&-&-&39.78\textcolor{gray}{\textsuperscript{7}}&33.80\textcolor{gray}{\textsuperscript{13}}&31.06\textcolor{gray}{\textsuperscript{14}}&-&-\\
HAT\citep{hat}&\textcolor{blue}{35.46}\textcolor{gray}{\textsuperscript{2}}&\textcolor{blue}{31.77}\textcolor{gray}{\textsuperscript{2}}&\textcolor{blue}{29.75}\textcolor{gray}{\textsuperscript{2}}&-&-&\textcolor{blue}{0.944}\textcolor{gray}{\textsuperscript{2}}&\textcolor{blue}{0.887}\textcolor{gray}{\textsuperscript{2}}&\textcolor{blue}{0.837}\textcolor{gray}{\textsuperscript{2}}&-&-&\textcolor{blue}{0.038}\textcolor{gray}{\textsuperscript{2}}&\textcolor{blue}{0.090}\textcolor{gray}{\textsuperscript{2}}&\textcolor{blue}{0.138}\textcolor{gray}{\textsuperscript{2}}&-&-&39.78\textcolor{gray}{\textsuperscript{7}}&33.91\textcolor{gray}{\textsuperscript{11}}&31.20\textcolor{gray}{\textsuperscript{12}}&-&-\\
HDSRNet\citep{hdsr}&34.64\textcolor{gray}{\textsuperscript{11}}&30.95\textcolor{gray}{\textsuperscript{13}}&29.04\textcolor{gray}{\textsuperscript{13}}&-&-&0.937\textcolor{gray}{\textsuperscript{13}}&0.873\textcolor{gray}{\textsuperscript{15}}&0.819\textcolor{gray}{\textsuperscript{15}}&-&-&0.043\textcolor{gray}{\textsuperscript{16}}&0.103\textcolor{gray}{\textsuperscript{15}}&0.154\textcolor{gray}{\textsuperscript{15}}&-&-&39.46\textcolor{gray}{\textsuperscript{12}}&34.20\textcolor{gray}{\textsuperscript{9}}&31.33\textcolor{gray}{\textsuperscript{8}}&-&-\\
GRLBase\citep{grl}&\textcolor{red}{35.66}\textcolor{gray}{\textsuperscript{1}}&\textcolor{red}{31.93}\textcolor{gray}{\textsuperscript{1}}&\textcolor{red}{29.91}\textcolor{gray}{\textsuperscript{1}}&-&-&\textcolor{red}{0.945}\textcolor{gray}{\textsuperscript{1}}&\textcolor{red}{0.889}\textcolor{gray}{\textsuperscript{1}}&\textcolor{red}{0.841}\textcolor{gray}{\textsuperscript{1}}&-&-&\textcolor{red}{0.037}\textcolor{gray}{\textsuperscript{1}}&\textcolor{red}{0.089}\textcolor{gray}{\textsuperscript{1}}&\textcolor{red}{0.135}\textcolor{gray}{\textsuperscript{1}}&-&-&39.99\textcolor{gray}{\textsuperscript{4}}&33.84\textcolor{gray}{\textsuperscript{12}}&31.12\textcolor{gray}{\textsuperscript{13}}&-&-\\
GRLSmall\citep{grl}&35.39\textcolor{gray}{\textsuperscript{4}}&31.65\textcolor{gray}{\textsuperscript{4}}&29.63\textcolor{gray}{\textsuperscript{5}}&-&-&0.943\textcolor{gray}{\textsuperscript{4}}&0.886\textcolor{gray}{\textsuperscript{4}}&0.835\textcolor{gray}{\textsuperscript{4}}&-&-&0.038\textcolor{gray}{\textsuperscript{4}}&0.092\textcolor{gray}{\textsuperscript{5}}&0.140\textcolor{gray}{\textsuperscript{4}}&-&-&39.54\textcolor{gray}{\textsuperscript{10}}&33.68\textcolor{gray}{\textsuperscript{14}}&31.03\textcolor{gray}{\textsuperscript{15}}&-&-\\
GRLTiny\citep{grl}&35.17\textcolor{gray}{\textsuperscript{6}}&31.41\textcolor{gray}{\textsuperscript{6}}&29.40\textcolor{gray}{\textsuperscript{6}}&-&-&0.942\textcolor{gray}{\textsuperscript{6}}&0.882\textcolor{gray}{\textsuperscript{6}}&0.830\textcolor{gray}{\textsuperscript{6}}&-&-&0.039\textcolor{gray}{\textsuperscript{6}}&0.096\textcolor{gray}{\textsuperscript{9}}&0.146\textcolor{gray}{\textsuperscript{7}}&-&-&39.20\textcolor{gray}{\textsuperscript{16}}&33.20\textcolor{gray}{\textsuperscript{15}}&30.56\textcolor{gray}{\textsuperscript{16}}&-&-\\
Bicubic&31.04\textcolor{gray}{\textsuperscript{18}}&28.25\textcolor{gray}{\textsuperscript{18}}&26.69\textcolor{gray}{\textsuperscript{18}}&24.87\textcolor{gray}{\textsuperscript{6}}&22.34\textcolor{gray}{\textsuperscript{6}}&0.893\textcolor{gray}{\textsuperscript{18}}&0.813\textcolor{gray}{\textsuperscript{18}}&0.752\textcolor{gray}{\textsuperscript{18}}&0.675\textcolor{gray}{\textsuperscript{6}}&0.587\textcolor{gray}{\textsuperscript{6}}&0.096\textcolor{gray}{\textsuperscript{18}}&0.191\textcolor{gray}{\textsuperscript{18}}&0.291\textcolor{gray}{\textsuperscript{18}}&0.439\textcolor{gray}{\textsuperscript{6}}&0.613\textcolor{gray}{\textsuperscript{6}}&32.79\textcolor{gray}{\textsuperscript{17}}&28.90\textcolor{gray}{\textsuperscript{17}}&26.57\textcolor{gray}{\textsuperscript{17}}&23.45\textcolor{gray}{\textsuperscript{6}}&18.74\textcolor{gray}{\textsuperscript{6}}\\
    \bottomrule
    \end{tabular}}
    \caption{Comparison of PSNR, SSIM \citep{SSIM}, LPIPS \citep{LPIPS} and FSDS metrics for different methods on the DIV2K dataset \citep{DIV2K}. Items with the highest and the second-highest mean values are highlighted in red and blue, respectively. The gray superscripts are the order of each method.}
    \label{tab:metrics}
\end{table*}
We conducted tests on the validation set of the DIV2K dataset \citep{DIV2K} using several methods  \citep{edsr, liif, OPE-SR, SRNO, LTE, rdn, swinir, ITSRN}, depicted in \Cref{tab:metrics}. The evaluation metrics include PSNR, SSIM  \citep{SSIM}, LPIPS \citep{LPIPS}, and our FSDS. All tests are performed using code and weights available in open-source official repositories. For all methods, we conduct experiment for $\times$2 to $\times$4. For methods that support arbitrary-scale ISR, we test for $\times$6 and $\times$12 as well. In the $\times$2 to $\times$4 range, GRLBase\citep{grl} consistently achieves the best performance across PSNR, SSIM and LPIPS metrics, and FSDS shows that SwinIR \citep{swinir} achieves the best performance. For $\times$6 and $\times$12, RDN-LTE  \citep{LTE} exhibits the best PSNR and SSIM metrics, while RDN-LIIF performs best on LPIPS and the FSDS metric. 

From \Cref{sec:fsds}, we claim that previous metrics are not sensitive to high-frequency information, while FSDS does. This can be proven by \Cref{tab:metrics}. In the case of slight high-frequency loss, such as on scales $\times$ 2 to $\times$ 4, FSDS responds differently compared to previous metrics. In cases that suffer from severe high-frequency loss, such as on $\times$ 6 and $\times$ 12 scales, FSDS shows consistency with previous metrics. This is because when high-frequency loss is slight, previous metrics fail to reflect such high-frequency loss and while the loss becomes more severe, they start to capture such loss. This observation shows the necessity of applying the FSDS metrics to assess image quality objectively.

\subsection{Some Exceptions to Impulse Responses}
\begin{figure}[h!]
    \centering
    \includegraphics[scale=0.20]{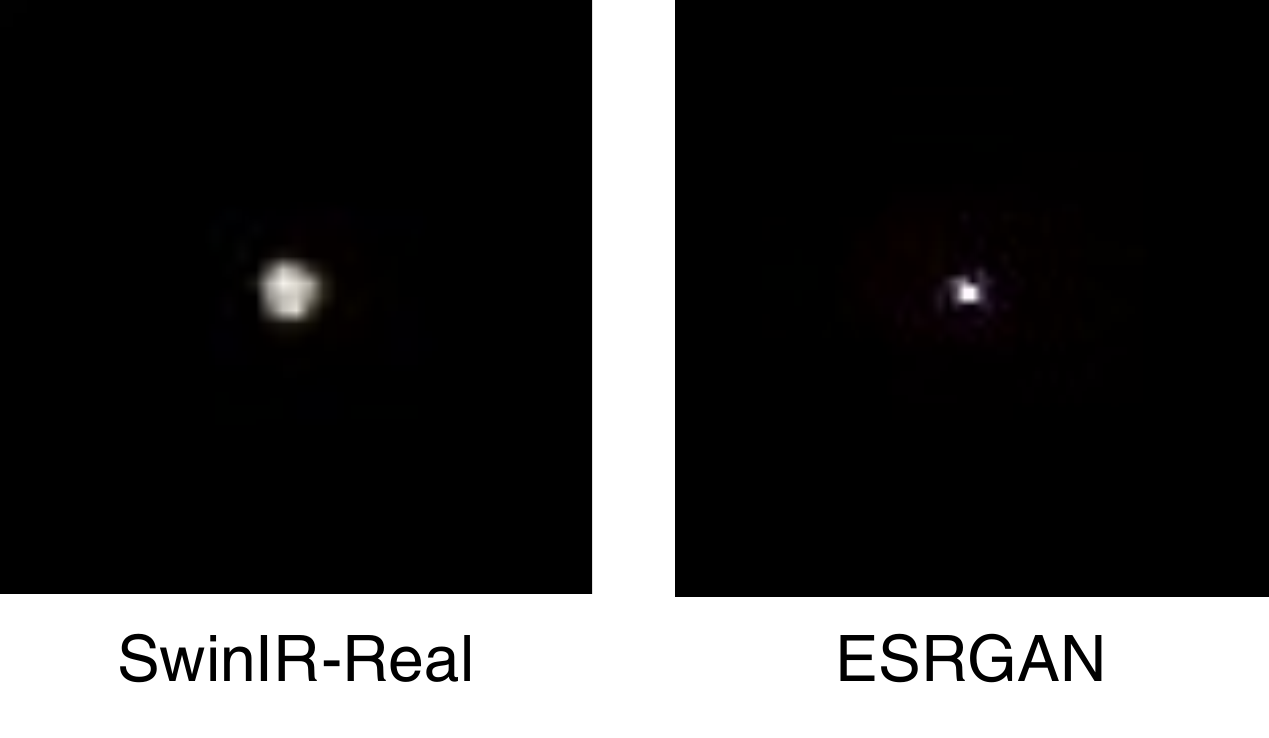}
    \caption{The impulse response of SwinIR-Real \citep{swinir} and ESRGAN \citep{ESRGAN} is not an obvious sinc function.}
    \label{fig:exceptions}
\end{figure}
We observe that not all impulse response of networks is `sinc' function, as shown in \Cref{fig:exceptions}. SwinIR-Real \citep{swinir} and ESRGAN \citep{ESRGAN} are trained using adversarial loss, while methods in \Cref{fig:sinc_impulse_response} uses loss like $\ell_1$ or $\ell_2$ loss. Therefore, we believe the `sinc' impulse response is related to the loss function.

\section{Conclusion}
In this paper, we report an intriguing observation. i.e., the sinc phenomenon, which reveals that the impulse response of ISR networks act as low-pass filters. Building on this observation, we introduce a novel approach called Hybrid Response Analysis (HyRA) to explore the hidden behavior of ISR networks. HyRA treats a neural network as a combination of a linear system and a non-linear system with a zero impulse response. The linear system functions as a low-pass filter, while the non-linear system utilizes prior knowledge to inject high-frequency details. To assess the neural network's information recovery across the frequency spectrum, we propose the Frequency Spectrum Distribution Similarity (FSDS) metric. FSDS uncovers properties overlooked by previous metrics, and experiments validate the rationality and necessity of it.
\newpage
\section*{Acknowledgements}
We appreciate anonymous reviewers for their previous suggestions to help this paper better.
Moreover, we would like to express our sincere gratitude to Ruijie Zhu (rzhu48@ucsc.edu) for his generous support in GPUs. Without his support, it is hard for us to do experiments using full-scale DIV2K dataset. This work is supported by NSFC
(12271083).
\section*{Impact Statement}
This paper presents work whose goal is to advance the interpretability of neural networks in Image super-resolution. There are many potential societal consequences of our work, none which we feel must be specifically highlighted here.

\nocite{langley00}

\bibliography{example_paper}

\begin{thebibliography}{39}
\providecommand{\natexlab}[1]{#1}
\providecommand{\url}[1]{\texttt{#1}}
\expandafter\ifx\csname urlstyle\endcsname\relax
  \providecommand{\doi}[1]{doi: #1}\else
  \providecommand{\doi}{doi: \begingroup \urlstyle{rm}\Url}\fi

\bibitem[Agustsson \& Timofte(2017)Agustsson and Timofte]{DIV2K}
Agustsson, E. and Timofte, R.
\newblock Ntire 2017 challenge on single image super-resolution: Dataset and study.
\newblock In \emph{Proceedings of the IEEE conference on computer vision and pattern recognition workshops}, pp.\  126--135, 2017.

\bibitem[Ahn et~al.(2018)Ahn, Kang, and Sohn]{CARN}
Ahn, N., Kang, B., and Sohn, K.-A.
\newblock Fast, accurate, and lightweight super-resolution with cascading residual network.
\newblock In \emph{Proceedings of the European conference on computer vision (ECCV)}, pp.\  252--268, 2018.

\bibitem[Bevilacqua et~al.(2012)Bevilacqua, Roumy, Guillemot, and Alberi-Morel]{set5}
Bevilacqua, M., Roumy, A., Guillemot, C., and Alberi-Morel, M.~L.
\newblock Low-complexity single-image super-resolution based on nonnegative neighbor embedding.
\newblock 2012.

\bibitem[Chen et~al.(2023)Chen, Wang, Zhou, Qiao, and Dong]{hat}
Chen, X., Wang, X., Zhou, J., Qiao, Y., and Dong, C.
\newblock Activating more pixels in image super-resolution transformer.
\newblock In \emph{Proceedings of the IEEE/CVF conference on computer vision and pattern recognition}, pp.\  22367--22377, 2023.

\bibitem[Chen et~al.(2021)Chen, Liu, and Wang]{liif}
Chen, Y., Liu, S., and Wang, X.
\newblock Learning continuous image representation with local implicit image function.
\newblock In \emph{Proceedings of the IEEE/CVF conference on computer vision and pattern recognition}, pp.\  8628--8638, 2021.

\bibitem[Gu \& Dong(2021)Gu and Dong]{chaodong}
Gu, J. and Dong, C.
\newblock Interpreting super-resolution networks with local attribution maps.
\newblock In \emph{Proceedings of the IEEE/CVF Conference on Computer Vision and Pattern Recognition}, pp.\  9199--9208, 2021.

\bibitem[Hu et~al.(2019)Hu, Mu, Zhang, Wang, Tan, and Sun]{metasr}
Hu, X., Mu, H., Zhang, X., Wang, Z., Tan, T., and Sun, J.
\newblock Meta-sr: A magnification-arbitrary network for super-resolution.
\newblock In \emph{Proceedings of the IEEE/CVF conference on computer vision and pattern recognition}, pp.\  1575--1584, 2019.

\bibitem[Huang et~al.(2015)Huang, Singh, and Ahuja]{Urban100}
Huang, J.-B., Singh, A., and Ahuja, N.
\newblock Single image super-resolution from transformed self-exemplars.
\newblock In \emph{Proceedings of the IEEE conference on computer vision and pattern recognition}, pp.\  5197--5206, 2015.

\bibitem[Hui et~al.(2019)Hui, Gao, Yang, and Wang]{IMDN}
Hui, Z., Gao, X., Yang, Y., and Wang, X.
\newblock Lightweight image super-resolution with information multi-distillation network.
\newblock In \emph{Proceedings of the 27th acm international conference on multimedia}, pp.\  2024--2032, 2019.

\bibitem[Lee \& Jin(2022)Lee and Jin]{LTE}
Lee, J. and Jin, K.~H.
\newblock Local texture estimator for implicit representation function.
\newblock In \emph{Proceedings of the IEEE/CVF conference on computer vision and pattern recognition}, pp.\  1929--1938, 2022.

\bibitem[Li et~al.(2023)Li, Fan, Xiang, Demandolx, Ranjan, Timofte, and Van~Gool]{grl}
Li, Y., Fan, Y., Xiang, X., Demandolx, D., Ranjan, R., Timofte, R., and Van~Gool, L.
\newblock Efficient and explicit modelling of image hierarchies for image restoration.
\newblock In \emph{Proceedings of the IEEE/CVF Conference on Computer Vision and Pattern Recognition}, pp.\  18278--18289, 2023.

\bibitem[Liang et~al.(2021)Liang, Cao, Sun, Zhang, Van~Gool, and Timofte]{swinir}
Liang, J., Cao, J., Sun, G., Zhang, K., Van~Gool, L., and Timofte, R.
\newblock Swinir: Image restoration using swin transformer.
\newblock In \emph{Proceedings of the IEEE/CVF international conference on computer vision}, pp.\  1833--1844, 2021.

\bibitem[Lim et~al.(2017)Lim, Son, Kim, Nah, and Mu~Lee]{edsr}
Lim, B., Son, S., Kim, H., Nah, S., and Mu~Lee, K.
\newblock Enhanced deep residual networks for single image super-resolution.
\newblock In \emph{Proceedings of the IEEE conference on computer vision and pattern recognition workshops}, pp.\  136--144, 2017.

\bibitem[Liu et~al.(2023)Liu, Li, Shang, Liu, Wan, Feng, and Timofte]{overview_arbsr}
Liu, H., Li, Z., Shang, F., Liu, Y., Wan, L., Feng, W., and Timofte, R.
\newblock Arbitrary-scale super-resolution via deep learning: A comprehensive survey.
\newblock \emph{Information Fusion}, pp.\  102015, 2023.

\bibitem[Miller et~al.(1992)Miller, Farison, and Shin]{miller1992spatially}
Miller, J.~W., Farison, J.~B., and Shin, Y.
\newblock Spatially invariant image sequences.
\newblock \emph{IEEE Transactions on Image Processing}, 1\penalty0 (2):\penalty0 148--161, 1992.

\bibitem[Nyquist(1928)]{nyquist}
Nyquist, H.
\newblock Certain topics in telegraph transmission theory.
\newblock \emph{Transactions of the American Institute of Electrical Engineers}, 1928.

\bibitem[Oppenheim \& Schafer(2009)Oppenheim and Schafer]{dsp_book}
Oppenheim, A.~V. and Schafer, R.~W.
\newblock \emph{Discrete-Time Signal Processing}.
\newblock Prentice Hall Press, USA, 2009.

\bibitem[Oppenheim et~al.(1996)Oppenheim, Willsky, and Nawab]{signalandsystem}
Oppenheim, A.~V., Willsky, A.~S., and Nawab, S.~H.
\newblock \emph{Signals \& Systems (2nd Ed.)}.
\newblock Prentice-Hall, Inc., USA, 1996.

\bibitem[Paszke et~al.(2019)Paszke, Gross, Massa, Lerer, Bradbury, Chanan, Killeen, Lin, Gimelshein, Antiga, et~al.]{Paszke2019PyTorchAI}
Paszke, A., Gross, S., Massa, F., Lerer, A., Bradbury, J., Chanan, G., Killeen, T., Lin, Z., Gimelshein, N., Antiga, L., et~al.
\newblock Pytorch: An imperative style, high-performance deep learning library.
\newblock \emph{Advances in neural information processing systems}, 32, 2019.

\bibitem[Song et~al.(2023)Song, Sun, Zhang, Su, Shi, and He]{OPE-SR}
Song, G., Sun, Q., Zhang, L., Su, R., Shi, J., and He, Y.
\newblock Ope-sr: Orthogonal position encoding for designing a parameter-free upsampling module in arbitrary-scale image super-resolution.
\newblock In \emph{Proceedings of the IEEE/CVF Conference on Computer Vision and Pattern Recognition}, pp.\  10009--10020, 2023.

\bibitem[Sundararajan et~al.(2017)Sundararajan, Taly, and Yan]{integrated_gradients}
Sundararajan, M., Taly, A., and Yan, Q.
\newblock Axiomatic attribution for deep networks.
\newblock In \emph{International conference on machine learning}, pp.\  3319--3328. PMLR, 2017.

\bibitem[Tian et~al.(2024)Tian, Zhang, Ren, Zuo, Zhang, and Lin]{hdsr}
Tian, C., Zhang, X., Ren, J., Zuo, W., Zhang, Y., and Lin, C.-W.
\newblock A heterogeneous dynamic convolutional neural network for image super-resolution.
\newblock \emph{arXiv preprint arXiv:2402.15704}, 2024.

\bibitem[Wang et~al.(2021)Wang, Wang, Lin, Yang, An, and Guo]{ArbSR}
Wang, L., Wang, Y., Lin, Z., Yang, J., An, W., and Guo, Y.
\newblock Learning a single network for scale-arbitrary super-resolution.
\newblock In \emph{Proceedings of the IEEE/CVF international conference on computer vision}, pp.\  4801--4810, 2021.

\bibitem[Wang et~al.(2018)Wang, Yu, Wu, Gu, Liu, Dong, Qiao, and Change~Loy]{ESRGAN}
Wang, X., Yu, K., Wu, S., Gu, J., Liu, Y., Dong, C., Qiao, Y., and Change~Loy, C.
\newblock Esrgan: Enhanced super-resolution generative adversarial networks.
\newblock In \emph{Proceedings of the European conference on computer vision (ECCV) workshops}, pp.\  0--0, 2018.

\bibitem[Wang et~al.(2023)Wang, Chen, Ni, Wang, Tong, and Liu]{EQSR}
Wang, X., Chen, X., Ni, B., Wang, H., Tong, Z., and Liu, Y.
\newblock Deep arbitrary-scale image super-resolution via scale-equivariance pursuit.
\newblock In \emph{Proceedings of the IEEE/CVF Conference on Computer Vision and Pattern Recognition}, pp.\  1786--1795, 2023.

\bibitem[Wang et~al.(2004)Wang, Bovik, Sheikh, and Simoncelli]{SSIM}
Wang, Z., Bovik, A.~C., Sheikh, H.~R., and Simoncelli, E.~P.
\newblock Image quality assessment: from error visibility to structural similarity.
\newblock \emph{IEEE transactions on image processing}, 13\penalty0 (4):\penalty0 600--612, 2004.

\bibitem[Wei \& Zhang(2023)Wei and Zhang]{SRNO}
Wei, M. and Zhang, X.
\newblock Super-resolution neural operator.
\newblock In \emph{Proceedings of the IEEE/CVF Conference on Computer Vision and Pattern Recognition}, pp.\  18247--18256, 2023.

\bibitem[Xu et~al.(2020)Xu, Qin, Sun, Wang, Chen, and Ren]{Xu_2020_CVPR}
Xu, K., Qin, M., Sun, F., Wang, Y., Chen, Y.-K., and Ren, F.
\newblock Learning in the frequency domain.
\newblock In \emph{Proceedings of the IEEE/CVF conference on computer vision and pattern recognition}, pp.\  1740--1749, 2020.

\bibitem[Xu(2018)]{xu1}
Xu, Z.~J.
\newblock Understanding training and generalization in deep learning by fourier analysis.
\newblock \emph{arXiv preprint arXiv:1808.04295}, 2018.

\bibitem[Xu(2020)]{Xu2020}
Xu, Z.-Q.~J.
\newblock Frequency principle: Fourier analysis sheds light on deep neural networks.
\newblock \emph{Communications in Computational Physics}, 28\penalty0 (5):\penalty0 1746--1767, 2020.

\bibitem[Xu et~al.(2019)Xu, Zhang, and Xiao]{Training_behavior}
Xu, Z.-Q.~J., Zhang, Y., and Xiao, Y.
\newblock Training behavior of deep neural network in frequency domain.
\newblock In \emph{Neural Information Processing: 26th International Conference, ICONIP 2019, Sydney, NSW, Australia, December 12--15, 2019, Proceedings, Part I 26}, pp.\  264--274. Springer, 2019.

\bibitem[Yang et~al.(2021)Yang, Shen, Yue, and Li]{ITSRN}
Yang, J., Shen, S., Yue, H., and Li, K.
\newblock Implicit transformer network for screen content image continuous super-resolution.
\newblock \emph{Advances in Neural Information Processing Systems}, 34:\penalty0 13304--13315, 2021.

\bibitem[Yang et~al.(2019)Yang, Zhang, Tian, Wang, Xue, and Liao]{overview_fix}
Yang, W., Zhang, X., Tian, Y., Wang, W., Xue, J.-H., and Liao, Q.
\newblock Deep learning for single image super-resolution: A brief review.
\newblock \emph{IEEE Transactions on Multimedia}, 21\penalty0 (12):\penalty0 3106--3121, 2019.

\bibitem[Young et~al.(2014)Young, Lai, Hodosh, and Hockenmaier]{Flickr2K}
Young, P., Lai, A., Hodosh, M., and Hockenmaier, J.
\newblock From image descriptions to visual denotations: New similarity metrics for semantic inference over event descriptions.
\newblock \emph{Transactions of the Association for Computational Linguistics}, 2:\penalty0 67--78, 2014.

\bibitem[Yu et~al.(2023)Yu, She, Liu, Cai, Shi, and Kwon]{Yu2023}
Yu, Y., She, K., Liu, J., Cai, X., Shi, K., and Kwon, O.
\newblock A super-resolution network for medical imaging via transformation analysis of wavelet multi-resolution.
\newblock \emph{Neural Networks}, 166:\penalty0 162--173, 2023.

\bibitem[Zhang et~al.(2018{\natexlab{a}})Zhang, Isola, Efros, Shechtman, and Wang]{LPIPS}
Zhang, R., Isola, P., Efros, A.~A., Shechtman, E., and Wang, O.
\newblock The unreasonable effectiveness of deep features as a perceptual metric.
\newblock In \emph{Proceedings of the IEEE conference on computer vision and pattern recognition}, pp.\  586--595, 2018{\natexlab{a}}.

\bibitem[Zhang et~al.(2018{\natexlab{b}})Zhang, Li, Li, Wang, Zhong, and Fu]{RCAN}
Zhang, Y., Li, K., Li, K., Wang, L., Zhong, B., and Fu, Y.
\newblock Image super-resolution using very deep residual channel attention networks.
\newblock In \emph{Proceedings of the European conference on computer vision (ECCV)}, pp.\  286--301, 2018{\natexlab{b}}.

\bibitem[Zhang et~al.(2018{\natexlab{c}})Zhang, Tian, Kong, Zhong, and Fu]{rdn}
Zhang, Y., Tian, Y., Kong, Y., Zhong, B., and Fu, Y.
\newblock Residual dense network for image super-resolution.
\newblock In \emph{Proceedings of the IEEE conference on computer vision and pattern recognition}, pp.\  2472--2481, 2018{\natexlab{c}}.

\bibitem[Zhang et~al.(2019)Zhang, Xu, Luo, and Ma]{zhang1}
Zhang, Y., Xu, Z.-Q.~J., Luo, T., and Ma, Z.
\newblock Explicitizing an implicit bias of the frequency principle in two-layer neural networks.
\newblock \emph{arXiv preprint arXiv:1905.10264}, 2019.

\end{thebibliography}
\bibliographystyle{icml2024}

\newpage
\appendix
\onecolumn
\appendix
\section{Notation Conventions}
\begin{table*}[h]
    \centering

    \begin{tabularx}{\textwidth}{p{0.15\textwidth}X}
        \toprule
        \underline{\textbf{\emph{Symbols}}} \\
        $j$&Imaginary number unit\\
        $*$&Convolution operator\\
        $I_{x,y}^{\mathrm{comment}}$&2-D signal with variant $x,y$\\
        $I_{j\omega_1,j\omega_2}^{\mathrm{comment}}$&Fourier transform of $I_{x,y}$\\
        $x(t)$&1-D signal with variant $t$\\
        $X(j\omega)$&Fourier transform of $x(t)$, $j\omega$ is a notation, $\omega$ is the variant\\
        $x[n]$&Discrete signal with index $n$\\
        $X[k]$&DFT of $x[n]$\\
        $\mathcal{F}[x(t)]$&Fourier transform operator, $X(j\omega)=\mathcal{F}[x(t)]$\\
        $\mathcal{F}^{-1}[X(j\omega)]$&Inverse Fourier transform, $x(t)=\mathcal{F}^{-1}[X(j\omega)]$\\
        \underline{\textbf{\emph{Signals}}} \\
        $\delta(t)$&Dirac $\delta$ function\\
        $sinc_{\omega}(t)$&The sinc funtion with parameter $\omega$, $sinc_{\omega}(t)=\frac{sin(\omega t)}{\pi t}$. The sinc function is the time-domain waveform of an ideal low-pass filter.\\
        $sinc^{\omega}_{x,y}$&2-D sinc function with parameter $\omega$, $sinc^{\omega}_{x,y}=\frac{sin (\omega x)}{\pi x}\cdot\frac{sin(\omega y)}{\pi y}$\\
        $s_{\Delta T}(t)$&1-D sample signal with a sample interval of $\Delta T$, 
 $s_{\Delta T}(t)=\sum\limits_{n=-\infty}^{\infty}\delta(t-nT)$\\
        \bottomrule
    \end{tabularx}
    \label{tab:symbols}
    \caption{Notation Conventions}
\end{table*}
\section{Signal Processing Theories}
We briefly introduce some related concepts and methods used in this paper in this section.
\subsection{System and Response\label{sec:system_and_response}}
The word `system' has many meanings and interpretations. This paper views a system as a process in which input signals are transformed by the system or cause the system to respond in some way, resulting in other signals as output \citep{signalandsystem}. Systems can be divided into linear systems and nonlinear systems according to their mathematical properties. A linear system refers to a system with such a property: the response of the system to the input $x_1(t)$, $x_2(t)$ is $y_1(t)$, $y_2(t)$ respectively, then when the input is $x_1(t)+x_2(t)$, the response of the system is $y_1(t)+y_2(t)$. 

Systems can also be divided into time-variant ones and time-invariant ones according to their temporal properties. A time-invariant system refers to that the properties of the system do not change with time, that is, the system has the same impulse response at any time. It satisfies such a relationship: when the input is $x(t)$, the output is $y(t)$, and when the input is $x(t-t_0)$, the output is $y(t-t_0)$.

A system with both linear and time-invariant properties is a linear time-invariant (LTI) system. For an LTI system, we can use `impulse response' to uniquely describe it: systems with the same impulse response are the same system, vice versa. The impulse response $h(t)$ is defined as the output of the system when the input signal is $\delta(t)$ (Dirac delta function). The response of a linear system to an arbitrary input signal can be computed through the convolution operation of its impulse response and the input signal, namely:
\begin{equation}
y(t)=x(t)*h(t)=\int_{-\infty}^{+\infty}x(\tau)h(t-\tau)d\tau.
\end{equation}
In the equation, $*$ is the convolution operator, $y(t)$ is the system output and $x(t)$ is the input signal. When we apply Fourier transform to the impulse response $h(t)$, then we can obtain the transfer function $H(j\omega)$ of the system. The transfer function describes the frequency domain waveform of the impulse response. According to the convolution theorem, the response of a linear system can also be obtained by multiplying the Fourier transform of the input signal by the transfer function of the system and then performing the inverse Fourier transform. In summary, given the impulse response of an LTI system, we can calculate the system's response to any output.

\subsection{Signal Sampling and Recovery\label{sec:samp_rec}}
\begin{figure}[h]
     \centering
     \includegraphics[scale=0.15]{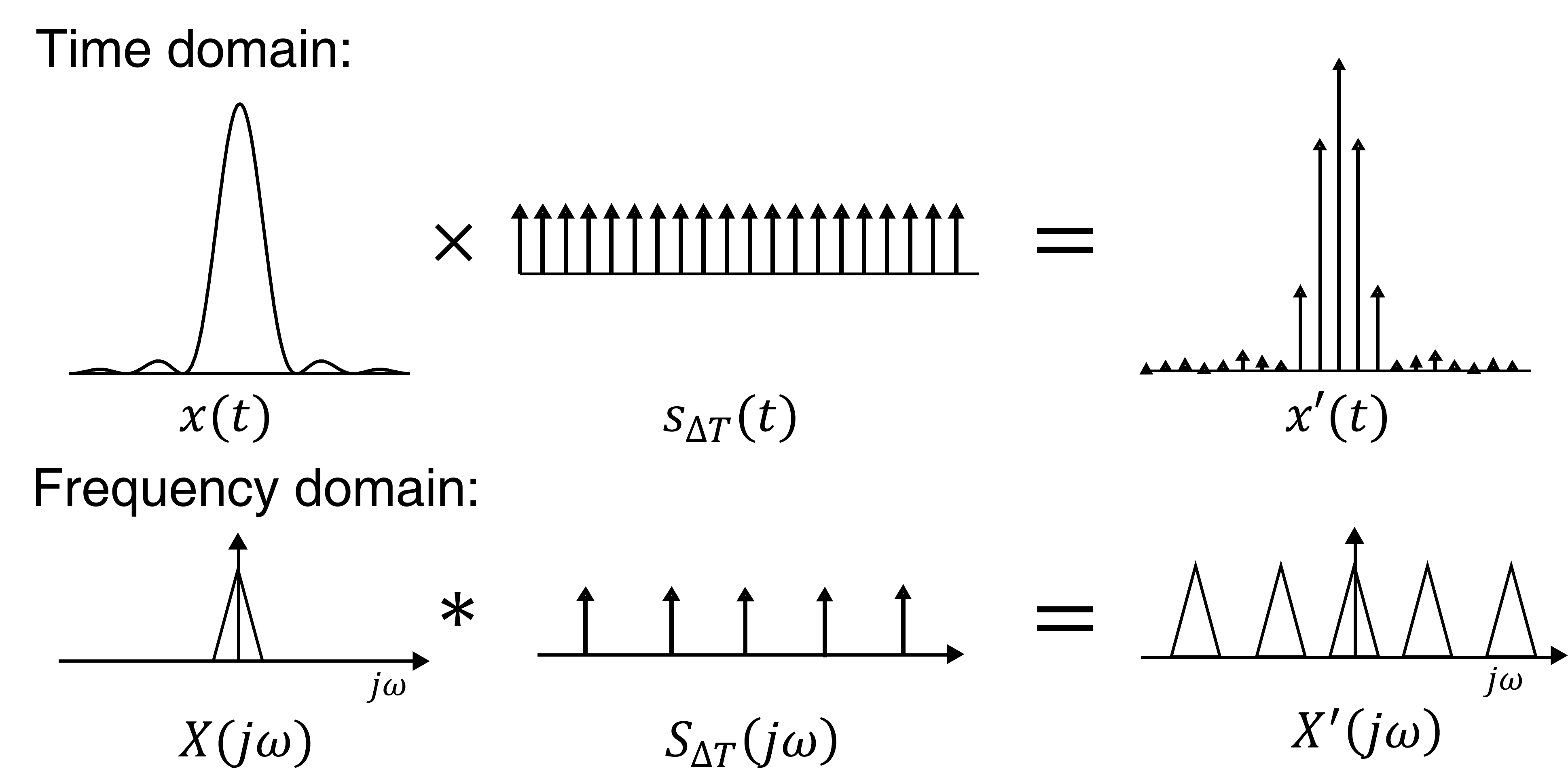}
     \caption{Time-domain to frequency-domain waveform variation of the continuous signal sampling process. The sampling function $s_{\delta T}$ is an impulse train sequence with an interval of $T$, and $S_{\delta T}(t)$ is its frequency domain waveform, which is also an impulse train sequence. Sampling a signal causes duplication in the frequency domain.}
     \label{fig:sampling}
\end{figure}
\begin{figure}[h]
     \centering
     \includegraphics[scale=0.15]{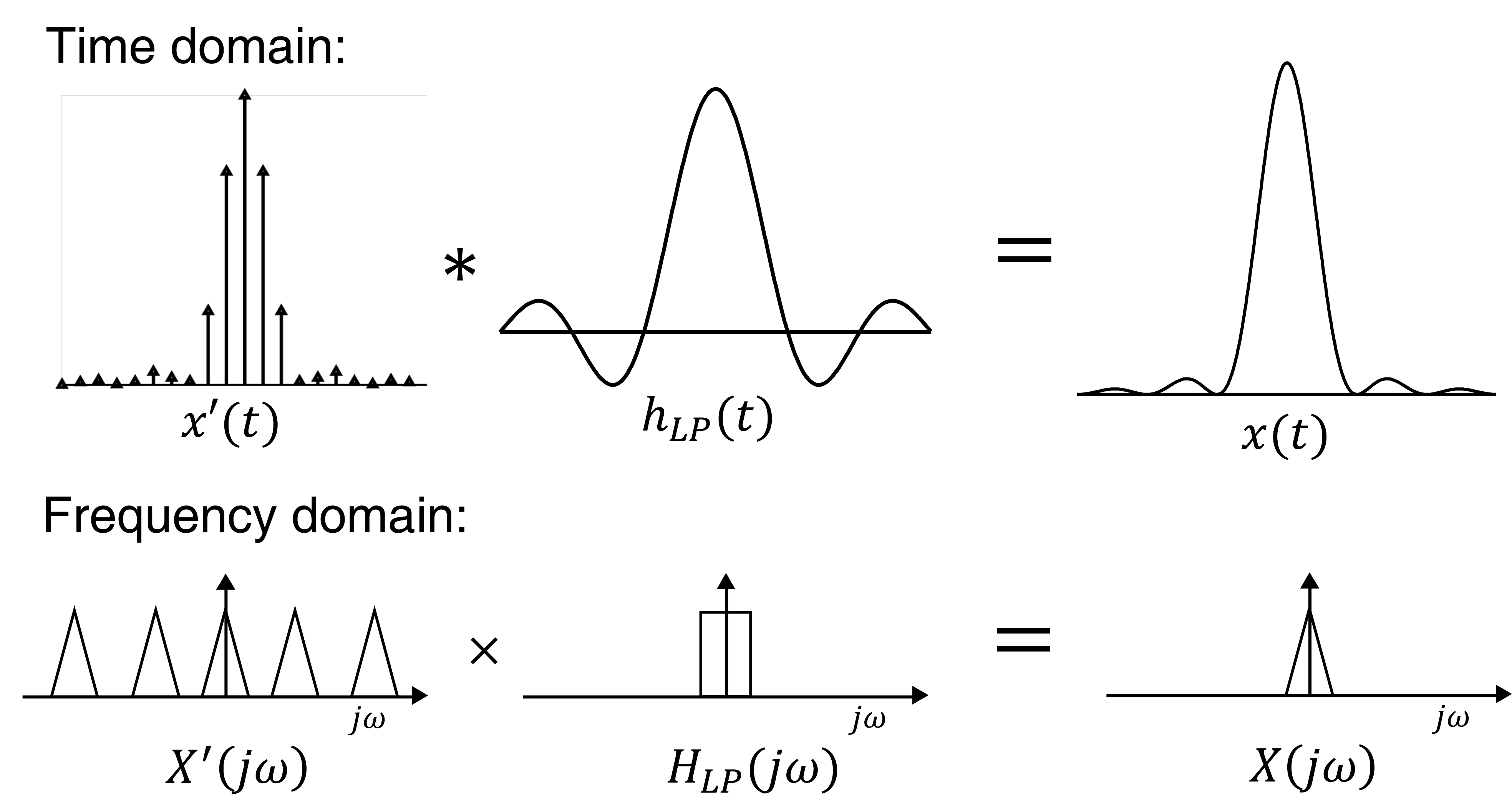}
     \caption{Time-domain to frequency-domain waveform variation in the process of sampling signal recovery. $h_{LP}(t)$ is the time-domain impulse response of a low-pass filter, and $H_{LP}(j\omega)$ is its frequency-domain waveform. }
     \label{fig:sap_rec}
\end{figure}
Signal sampling and sample recovery are very common operations, and in this section, we will briefly analyze this process from the perspective of both the time-domain and frequency-domain. The upper part of \Cref{fig:sampling} shows the time domain waveform variation of signal sampling process, and the lower part shows the frequency domain waveform variation of signal sampling process. To sample a continuous signal, the sampling process can be regarded as the multiplication of the original signal $x(t) $and an impulse train signal $s_{\Delta T}(t)$. It can be described as:
\begin{equation}
     \left\{\begin{array}{c}
          x'(t)=x(t)\cdot s_{\Delta T}(t),\\
         s_{\Delta T}(t)=\sum\limits_{n=-\infty}^{\infty}\delta(t-nT),
     \end{array}\right.
\end{equation}
where $T$ denotes the sampling interval, $x'(t)$ denotes the sampled signal. According to the convolution theorem, the frequency domain change of the sampling process can be described in the following way:
\begin{equation}
     \begin{aligned}
         X'(j\omega)&=X(j\omega)*S_{\delta T}(j\omega)\\
         &=\sum\limits_{k=-\infty}^{\infty}X[j(\omega-k\frac{2\pi}{T})].
     \end{aligned}
\end{equation}
\textbf{That is, the sampling process is reflected in the frequency spectrum as a periodic extension of the frequency spectrum of the original signal $x(t)$. }
\par
\Cref{fig:sap_rec} shows the time domain and frequency domain waveform variation during the recovery process. For sampling recovery, in order to restore the sampled signal $x'(t)$ to the original signal $x(t)$, from the perspective of frequency domain, a low-pass filter is all we needed, that is, convolving the sampled signal with a low-pass filter $h_{LP}(t)$. This process can be expressed as:
\begin{equation}
     \begin{aligned}
         x(t)&=x'(t)*h_{LP}(t)\\
         &=x'(t)*sinc_\omega(t),
     \end{aligned}
\end{equation}
where in the equation, $h_{LP}(t)=sinc_{\omega_0}(t)=\frac{sin(\omega_0t)}{\pi t}$ is the time domain response of the ideal low-pass filter, and its frequency-domain waveform $H_{LP }(j\omega)$ is a rectangular window.

\subsection{Spectrum Aliasing\label{sec:freq_overlap}}
\begin{figure*}[h]
     \centering
     \resizebox{\textwidth}{!}{
     \includegraphics[scale=0.12]{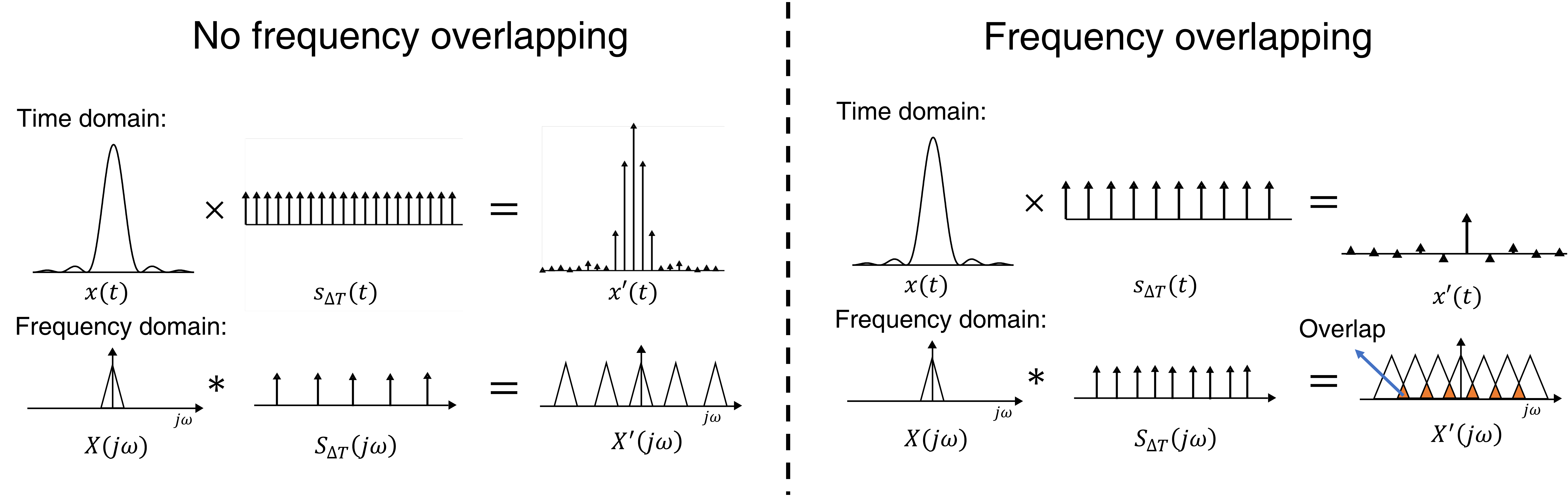}}
     \caption{The illustration of spectrum aliasing. On the left, there is no aliasing as the sampling rate is sufficiently high. On the right, aliasing occurs due to an insufficient sampling rate. }
     \label{fig:freq-overlapping}
\end{figure*}
Spectrum aliasing is a manifestation of information loss. \Cref{fig:freq-overlapping} depicts the time-domain and frequency-domain scenarios of no frequency overlapping and frequency overlapping, respectively. When the sampling rate is lower than the Nyquist sampling rate\footnote{The minimum sampling rate that can completely restore the origin of the sampled from sampled signal, which is twice the highest frequency of the original signal} \citep{nyquist}. When the sampling rate is below the Nyquist sampling rate, the approach mentioned in \Cref{sec:samp_rec} cannot completely restore the original signal $x(t)$. From \Cref{tab:ft-pair}, we can see that for the sample signal $s_{\Delta T}(t)$, the larger $T$ is, the sparser its time domain impulse train gets, while in the frequency spectrum the impulse trains gets denser. When the impulse trains in the frequency domain become sufficiently dense, and the spectrum of the original signal is periodically extended, overlapping occurs, preventing the complete recovery of the original signal. In ISR tasks, spectrum aliasing is manifested when restoring a low-resolution image to a high-resolution image, resulting in the loss of high-frequency information such as details and textures.
\section{Extra Experiments\label{sec:extra_exp}}
\subsection{Linear and Non-linear Responses}
\begin{figure}[h]
    \centering
    \includegraphics[width=\linewidth]{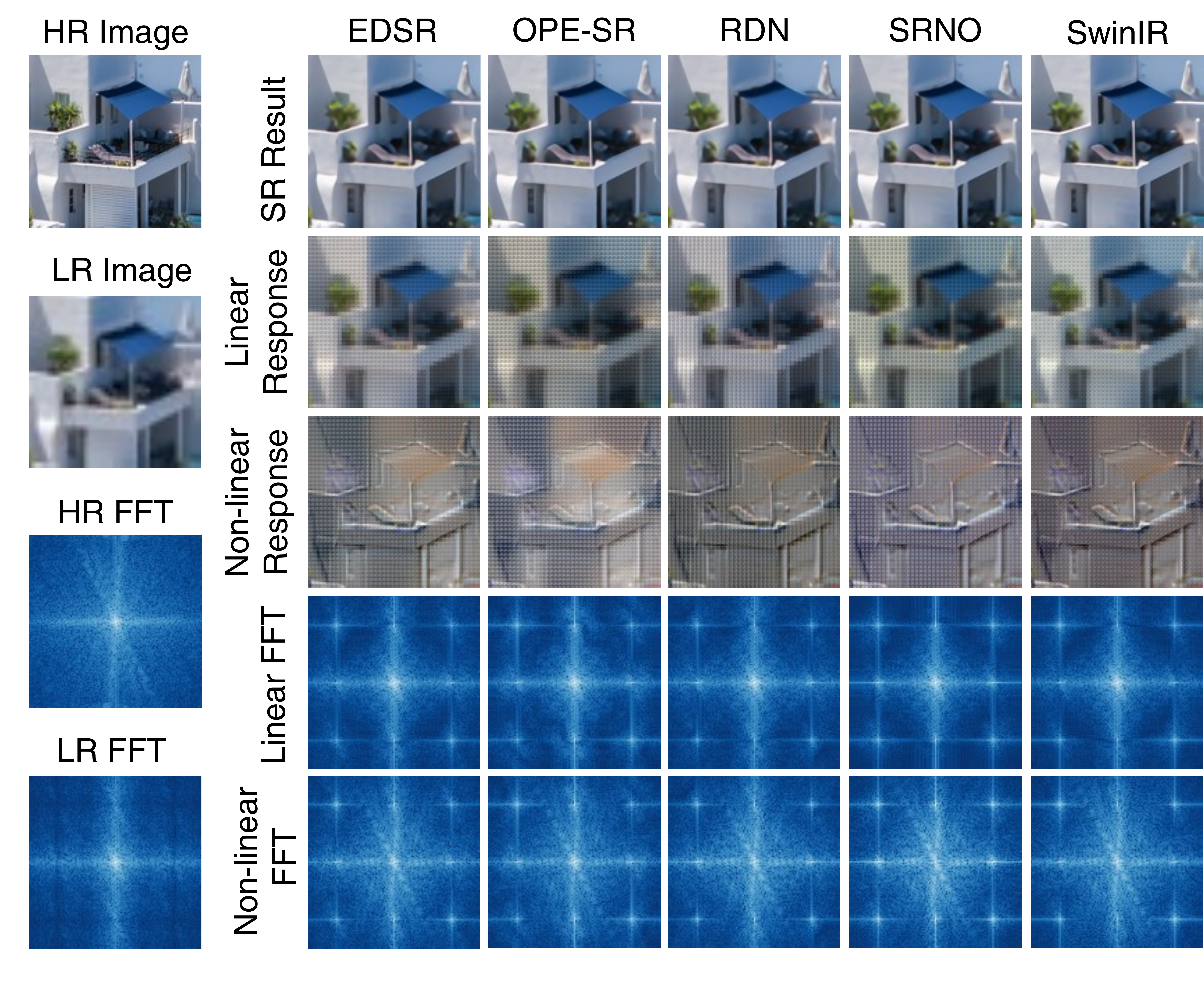}
    \caption{Linear and non-linear responses and their corresponding frequency spectrum of various ISR methods.}
    \label{fig:demonstrate}
\end{figure}
In \Cref{fig:demonstrate}, we present the linear and nonlinear responses of various ISR networks along with their corresponding spectrums. From the figure, it is evident that different networks exhibit varying filtering effects in their linear components. EDSR demonstrates a pronounced removal of high-frequency components, and compared to other methods, it exhibits the smallest area of brightness diffusion around the central bright spot in its spectrum. From the nonlinear responses, it can be observed that the nonlinear components of the networks are all involved in supplementing high-frequency information and correcting distortions.
\subsection{Space Invariance}
\begin{figure}[h]
    \centering
    \includegraphics[scale=0.15]{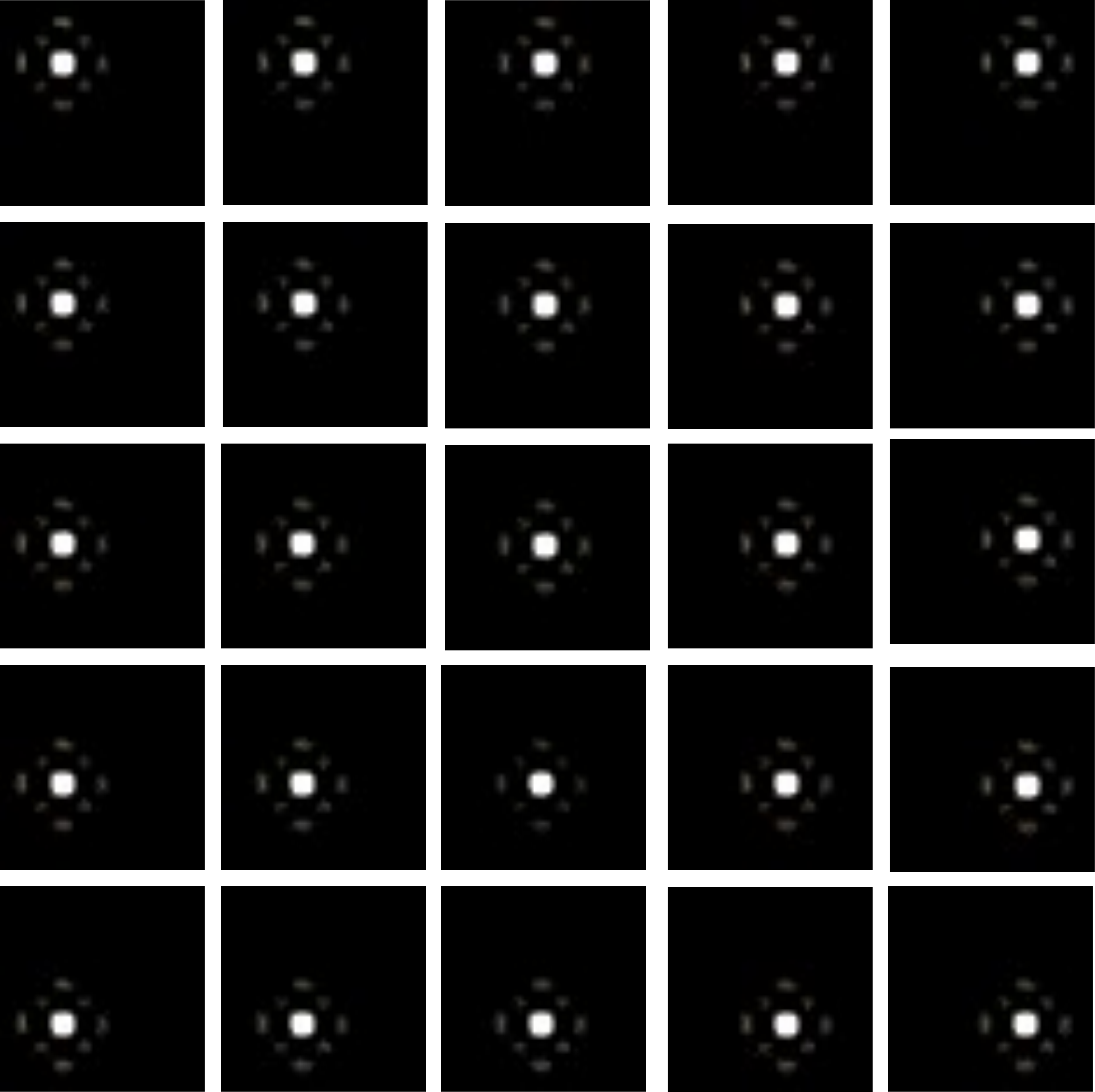}
    \caption{Spatial invariance experiment conducted on SwinIR \citep{swinir}. When we feed the SwinIR network with impulses at various positions, the ISR results demonstrate that the RDN exhibits spatial invariance.}
    \label{fig:space_invariance}
\end{figure}
We conducted spatial invariance testing on RDN  \citep{rdn} (For the concept of spatial invariance, please refer to \Cref{sec:HyRA}). The input data consists of an image where only one pixel is white (the pixel value is 1), and all other pixels are black (the pixel value is 0). By shifting the position of this white pixel, we obtain $I(x-\Delta x,y-\Delta y)$. This shifted input $I(x-\Delta x,y-\Delta y)$ is then fed into the neural network, and we obtain its shifted impulse response, as illustrated in \Cref{fig:space_invariance}. Observing the experimental results, we find that the responses to different $I(x-\Delta x,y-\Delta y)$ are consistent, with the only difference being their position. This demonstrates that, for ISR networks, the linear component in HyRA exhibits spatial invariance.
\subsection{Exploration of the Positional Origin of Sinc-like Patterns}
\begin{figure*}[ht!]
    \centering
    \includegraphics[scale=0.12]{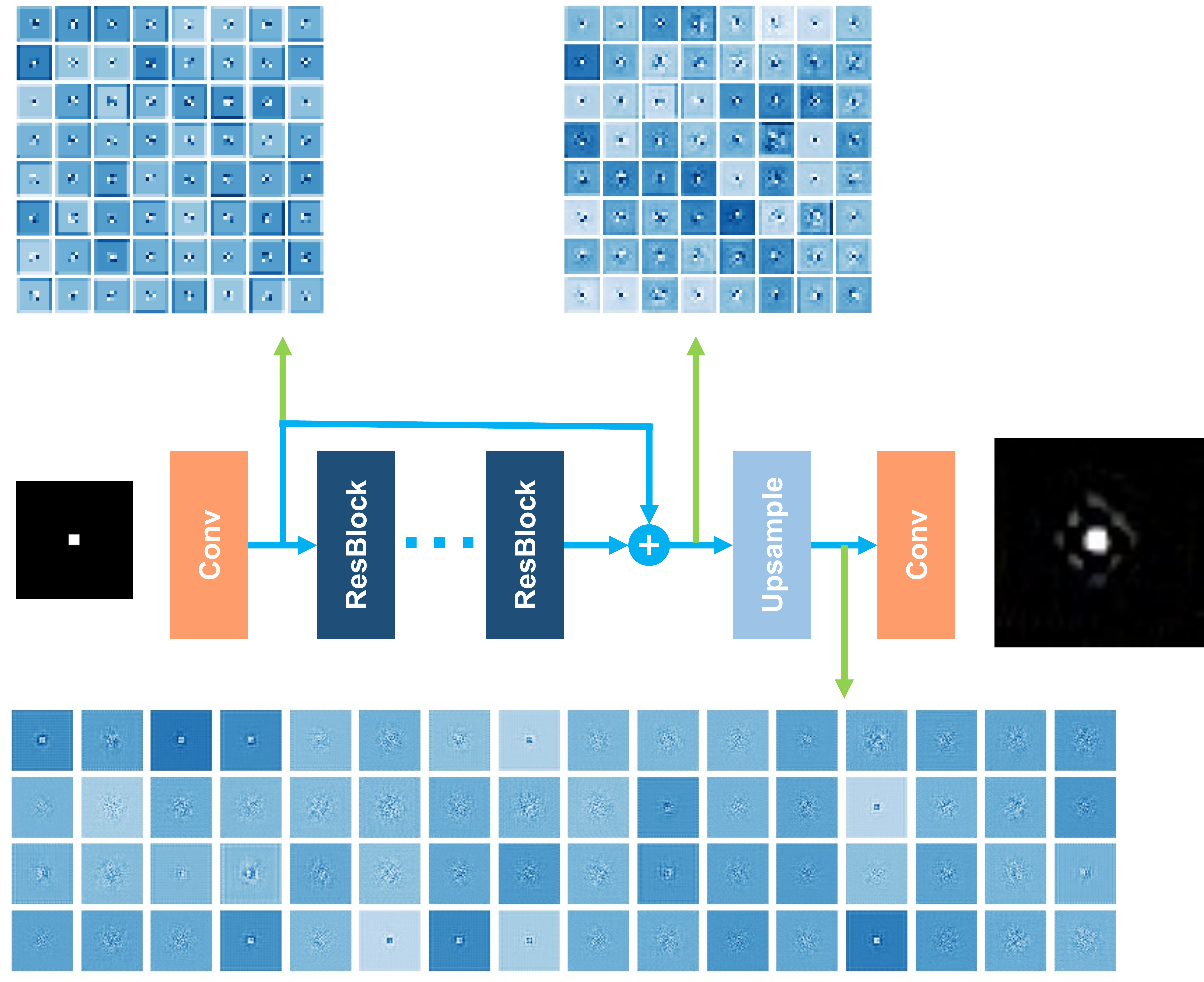}
    \caption{Visualization of feature maps in different EDSR  \citep{edsr} network layers. The sinc-like pattern start to take shapes after the sub-pixel convolution, before the last convolution layer.}
    \label{fig:sinc form}
\end{figure*}
As shown in \Cref{fig:sinc form}, we visualize the output features of different components in the EDSR  \citep{edsr} network for analysis. We observe that the approximate shape of the sinc function begins to take form after the Upsampler module, and after a convolution, it essentially forms the shape of a sinc function. Interestingly, in the EDSR network, the Upsample module uses sub-pixel convolution (convolution + pixel shuffle) for upsampling without any interpolation. This indicates that the low-pass filter present in the network is learned by the network itself and not introduced by interpolation kernels.
\section{The Fourier Transform Pairs Involved in This Paper}
\begin{table*}[h]
    \centering
    \renewcommand{\arraystretch}{2.5} 
    \begin{tabularx}{\textwidth}{XlXl}
    \toprule
    Symbol/Name&Section(s)&Time domain& Frequency Domain  \\
    \midrule
    $s_{\Delta T}(t)$, $\Delta T$ is the samping interval&\Cref{sec:samp_rec}, \Cref{sec:freq_overlap}&$\sum\limits_{n=-\infty}^{\infty}\delta(t-nT)$&$\frac{2\pi}{T}\sum\limits_{k=-\infty}^{\infty}\delta(\omega-k\frac{2\pi}{T})$\\
    Ideal Low-pass filter&\Cref{sec:hi_is_lpfing}&$x(t)=sinc_{\omega_0}(t)=\frac{sin(\omega_0t)}{\pi t}$, $\omega_0$ is called the cut-off frequency &$X(j\omega)=\begin{cases}
        1,&|\omega|<\omega_0  \\
         0,&|\omega |>\omega_0
    \end{cases}$\\
$\delta(t)$&\Cref{sec:hi_is_lpfing}&$\lim\limits_{\tau\rightarrow 0}\int_{-\tau}^{+\tau}\delta(t)=1$&1\\
         
    \bottomrule
    \end{tabularx}
    \caption{Fourier transform pairs}
    \label{tab:ft-pair}
\end{table*}

\section{The Windowing Operation\label{sec:app_windowing}}
\begin{figure}[h]
    \centering
    \includegraphics[scale=0.2]{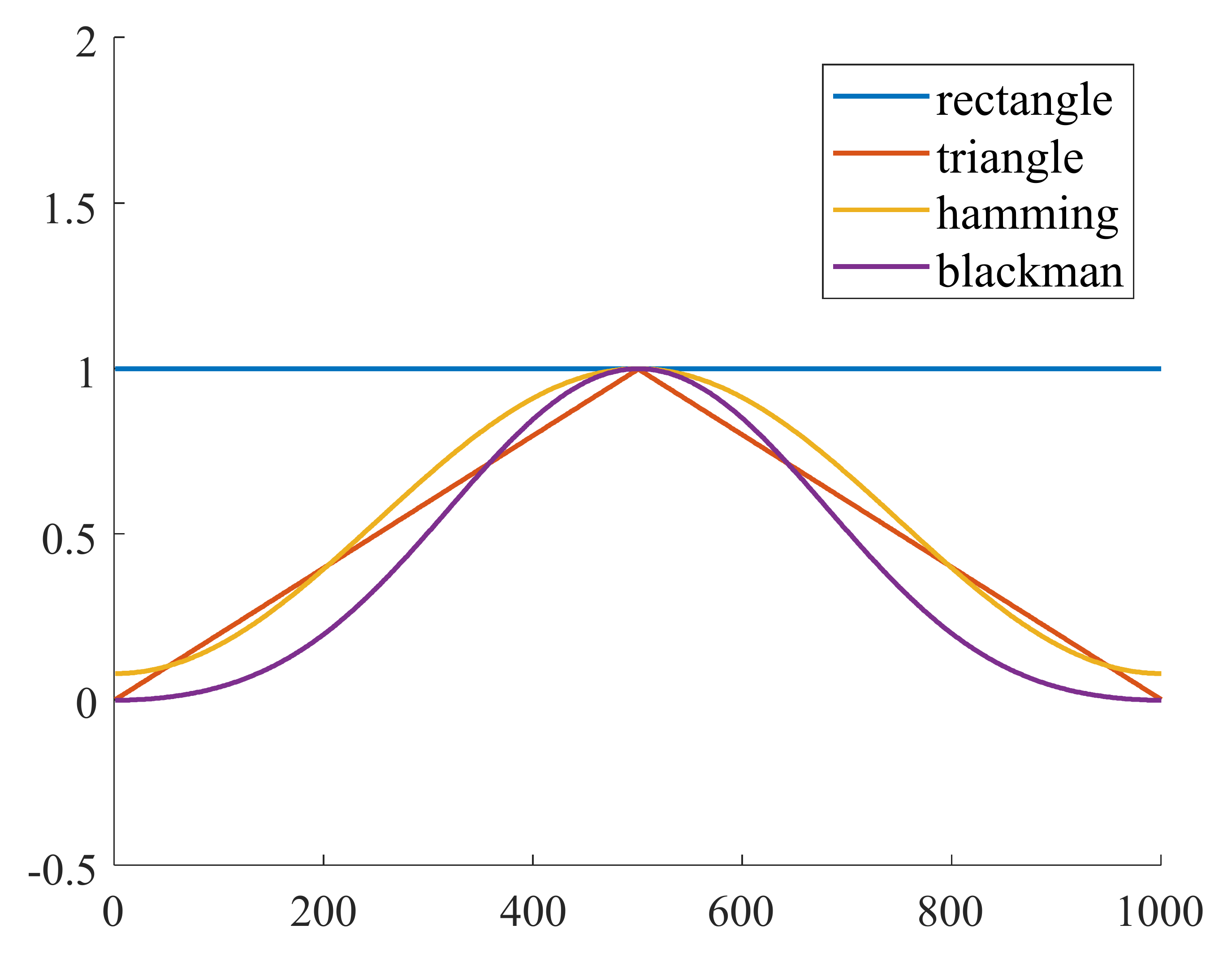}
    \caption{Various window functions.}
    \label{fig:VariousWindows}
\end{figure}
The time-domain waveform of an ideal low-pass filter is a sinc function. The sinc function is defined over $[-\infty, \infty]$, and the number of zero crossings is countable. This implies that in reality, an ideal low-pass filter does not exist. In discrete-time signal processing, truncating a designed filter using a window function is common. There are many window functions, such as the rectangular window, Hanning window, Blackman window, and so on. \Cref{fig:VariousWindows} illustrates some commonly used window functions. Observing the experimental results and analyzing the relationship between the peak values of the main lobe and the first side lobe, we find that the impulse response of the neural network seems to undergo windowing. However, different networks appear to adopt different window functions.

\section{Frequency Spectrum Period Extension Caused by Zero Padding \label{sec:freq_spec_reduplicate}}
\begin{figure}[h]
    \centering
    \includegraphics[scale=0.13]{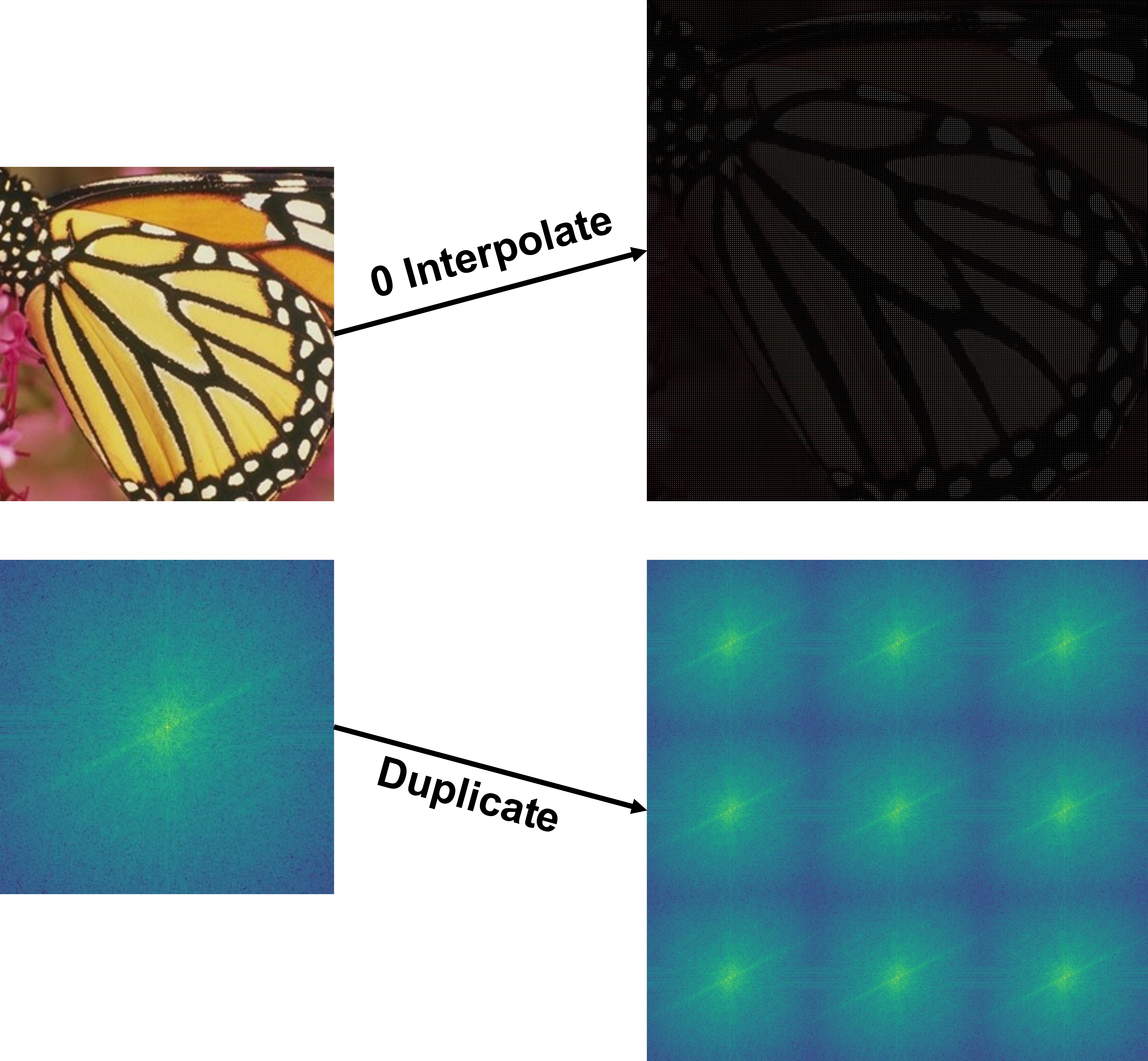}
    \caption{Performing zero-padding on an image to reach the target size will result in periodic extension in the frequency spectrum obtained through its Discrete Fourier Transform.}
    \label{fig:enter-label}
\end{figure}
When considering integer factor ISR, our approach to computing the linear component response is as follows: first, upsample the low-resolution image to the high-resolution image through zero-padding, and then convolve it with the impulse response to obtain the response. During the zero-padding process, it leads to period extension in the frequency spectrum. For a signal $x[n]$ of length $N$ undergoing DFT to obtain $X[k]$, we have:
\begin{equation}
    X[k]=\sum\limits_{n=0}^{N-1}x[n]e^{-j\frac{2\pi}{N}kn}.
\end{equation}
Then, zero-padding is applied to $x[n]$, producing in a new signal $x_2[n]$ of length $3N$:
\begin{equation}
    x_2[n]=\left\{\begin{array}{ll}
         x[\frac{n}{3}],&n=0,3,\cdots,3N-3 \\
         0,&otherwise.
    \end{array}\right.
\end{equation}
Perform DFT to $x_2[n]$ to obtain $X'[k]$, then we have:
\begin{equation}
    \begin{aligned}
    X'[k]&=\sum\limits_{n=0}^{3N-1}x_2[n]e^{-j\frac{2\pi}{3N}k\cdot n}\\
    &=\sum\limits_{n=0}^{N-1}x[n]e^{-j\frac{2\pi}{3N}k\cdot 3n}\\
    &=\sum\limits_{n=0}^{N-1}x[n]e^{-j\frac{2\pi}{N}kn}\\ 
    \end{aligned}.
\end{equation}
When $k<N$, there exists:
\begin{equation}
    e^{-j\frac{2\pi}{N}kn}=e^{-j\frac{2\pi}{N}(k+N)n}=e^{-j\frac{2\pi}{N}(k+2N)n}=\cdots.
\end{equation}
Therefore,
\begin{equation}
    X'[k]=\left\{\begin{array}{ll}
         X[k]& 0\leq k<N\\
         X[k\ mod\ N]& N\leq k <3N-1
    \end{array}\right.
\end{equation}
Thus, zero-padding causes period extension in the frequency spectrum. Ideally, the extended spectrum would be filtered out by the low-pass filter used in the ISR process. However, due to the limited filtering capability of the potential filters within the neural network, the stopband attenuation is low, and the extended spectrum cannot be completely filtered out.

\section{Comparison Between FSDS and $\ell_1$, $\ell_2$ norms}

We compare our proposed FSDS metric with $\ell_1$ norm, and $\ell_2$ norm on both frequency domain and image domain as depicted in \Cref{tab:l1} and \Cref{tab:l2}. From these two figures, we can observe that $\ell_1$ norm and $\ell_2$ norm produces similar ranking orders when they are calculated on the same domain, indicating that $\ell_1$ norm is equivalent to $\ell_2$ norm when assessing image quality. However, the ranking orders produced by FSDS is distinctive to that of $\ell_1$ and $\ell_2$. This means our FSDS metric reflects image quality in a unique way.




\begin{table}[h]
    \centering
\resizebox{\textwidth}{!}{
    \begin{tabular}{lcccccccccccccccc}
    \toprule 
         \multirow{3}*{Method}&\multicolumn{5}{c}{{FSDS (Ours)}}&\multicolumn{5}{c}{{$\ell_1$ Norm in Frequency Domain}}&\multicolumn{5}{c}{{$\ell_1$ Norm}}\\
         \cmidrule(l{2pt}r{2pt}){2-6}\cmidrule(l{2pt}r{2pt}){7-11}\cmidrule(l{2pt}r{2pt}){12-16}&{$\times 2$}&{$\times 3$}&{$\times 4$}&{$\times 6$}&{$\times 12$}&{$\times 2$}&{$\times 3$}&{$\times 4$}&{$\times 6$}&{$\times 12$}&{$\times 2$}&{$\times 3$}&{$\times 4$}&{$\times 6$}&{$\times 12$}\\
         
         \midrule
         EDSR\citep{edsr}&39.210\textcolor{gray}{\textsuperscript{15}}&34.148\textcolor{gray}{\textsuperscript{10}}&31.380\textcolor{gray}{\textsuperscript{7}}&-&-&35.190\textcolor{gray}{\textsuperscript{12}}&51.676\textcolor{gray}{\textsuperscript{12}}&62.478\textcolor{gray}{\textsuperscript{16}}&-&-&0.012\textcolor{gray}{\textsuperscript{10}}&0.019\textcolor{gray}{\textsuperscript{13}}&0.023\textcolor{gray}{\textsuperscript{16}}&-&-\\
EDSR-LIIF\citep{edsr}&39.371\textcolor{gray}{\textsuperscript{13}}&34.535\textcolor{gray}{\textsuperscript{6}}&31.321\textcolor{gray}{\textsuperscript{9}}&28.446\textcolor{gray}{\textsuperscript{4}}&23.147\textcolor{gray}{\textsuperscript{3}}&35.369\textcolor{gray}{\textsuperscript{14}}&51.886\textcolor{gray}{\textsuperscript{14}}&61.996\textcolor{gray}{\textsuperscript{15}}&73.963\textcolor{gray}{\textsuperscript{4}}&88.203\textcolor{gray}{\textsuperscript{5}}&0.013\textcolor{gray}{\textsuperscript{15}}&0.019\textcolor{gray}{\textsuperscript{15}}&0.023\textcolor{gray}{\textsuperscript{15}}&0.029\textcolor{gray}{\textsuperscript{4}}&0.042\textcolor{gray}{\textsuperscript{4}}\\
EDSR-OPESR\citep{edsr}&39.798\textcolor{gray}{\textsuperscript{6}}&34.630\textcolor{gray}{\textsuperscript{5}}&31.286\textcolor{gray}{\textsuperscript{11}}&-&-&36.796\textcolor{gray}{\textsuperscript{16}}&51.948\textcolor{gray}{\textsuperscript{15}}&61.711\textcolor{gray}{\textsuperscript{14}}&-&-&0.013\textcolor{gray}{\textsuperscript{16}}&0.018\textcolor{gray}{\textsuperscript{11}}&0.023\textcolor{gray}{\textsuperscript{12}}&-&-\\
EDSR-SRNO\citep{edsr}&39.533\textcolor{gray}{\textsuperscript{11}}&34.527\textcolor{gray}{\textsuperscript{7}}&31.448\textcolor{gray}{\textsuperscript{6}}&28.458\textcolor{gray}{\textsuperscript{3}}&22.778\textcolor{gray}{\textsuperscript{4}}&34.688\textcolor{gray}{\textsuperscript{10}}&51.113\textcolor{gray}{\textsuperscript{10}}&61.140\textcolor{gray}{\textsuperscript{11}}&73.136\textcolor{gray}{\textsuperscript{3}}&87.704\textcolor{gray}{\textsuperscript{3}}&0.012\textcolor{gray}{\textsuperscript{11}}&0.018\textcolor{gray}{\textsuperscript{10}}&0.023\textcolor{gray}{\textsuperscript{10}}&0.029\textcolor{gray}{\textsuperscript{3}}&0.041\textcolor{gray}{\textsuperscript{3}}\\
EDSR-LTE\citep{edsr}&39.290\textcolor{gray}{\textsuperscript{14}}&34.328\textcolor{gray}{\textsuperscript{8}}&31.303\textcolor{gray}{\textsuperscript{10}}&-&-&35.110\textcolor{gray}{\textsuperscript{11}}&51.556\textcolor{gray}{\textsuperscript{11}}&61.706\textcolor{gray}{\textsuperscript{12}}&-&-&0.013\textcolor{gray}{\textsuperscript{14}}&0.019\textcolor{gray}{\textsuperscript{14}}&0.023\textcolor{gray}{\textsuperscript{14}}&-&-\\
RDN\citep{rdn}&40.022\textcolor{gray}{\textsuperscript{3}}&32.946\textcolor{gray}{\textsuperscript{16}}&31.646\textcolor{gray}{\textsuperscript{4}}&-&-&34.595\textcolor{gray}{\textsuperscript{9}}&53.431\textcolor{gray}{\textsuperscript{16}}&61.069\textcolor{gray}{\textsuperscript{10}}&-&-&0.013\textcolor{gray}{\textsuperscript{12}}&0.019\textcolor{gray}{\textsuperscript{16}}&0.023\textcolor{gray}{\textsuperscript{11}}&-&-\\
RDN-LIIF\citep{rdn}&39.690\textcolor{gray}{\textsuperscript{9}}&34.831\textcolor{gray}{\textsuperscript{3}}&31.832\textcolor{gray}{\textsuperscript{3}}&\textcolor{red}{28.894}\textcolor{gray}{\textsuperscript{1}}&\textcolor{red}{23.778}\textcolor{gray}{\textsuperscript{1}}&34.224\textcolor{gray}{\textsuperscript{8}}&50.434\textcolor{gray}{\textsuperscript{8}}&60.563\textcolor{gray}{\textsuperscript{9}}&\textcolor{blue}{72.753}\textcolor{gray}{\textsuperscript{2}}&\textcolor{blue}{87.488}\textcolor{gray}{\textsuperscript{2}}&0.012\textcolor{gray}{\textsuperscript{8}}&0.018\textcolor{gray}{\textsuperscript{9}}&0.022\textcolor{gray}{\textsuperscript{9}}&\textcolor{blue}{0.029}\textcolor{gray}{\textsuperscript{2}}&\textcolor{blue}{0.041}\textcolor{gray}{\textsuperscript{2}}\\
RDN-OPESR\citep{rdn}&\textcolor{blue}{40.188}\textcolor{gray}{\textsuperscript{2}}&\textcolor{blue}{34.959}\textcolor{gray}{\textsuperscript{2}}&31.475\textcolor{gray}{\textsuperscript{5}}&-&-&36.126\textcolor{gray}{\textsuperscript{15}}&50.792\textcolor{gray}{\textsuperscript{9}}&60.388\textcolor{gray}{\textsuperscript{8}}&-&-&0.013\textcolor{gray}{\textsuperscript{13}}&0.018\textcolor{gray}{\textsuperscript{7}}&0.022\textcolor{gray}{\textsuperscript{7}}&-&-\\
RDN-LTE\citep{rdn}&39.825\textcolor{gray}{\textsuperscript{5}}&34.749\textcolor{gray}{\textsuperscript{4}}&\textcolor{blue}{31.837}\textcolor{gray}{\textsuperscript{2}}&\textcolor{blue}{28.654}\textcolor{gray}{\textsuperscript{2}}&\textcolor{blue}{23.346}\textcolor{gray}{\textsuperscript{2}}&33.997\textcolor{gray}{\textsuperscript{7}}&50.106\textcolor{gray}{\textsuperscript{7}}&60.229\textcolor{gray}{\textsuperscript{7}}&\textcolor{red}{72.301}\textcolor{gray}{\textsuperscript{1}}&\textcolor{red}{87.090}\textcolor{gray}{\textsuperscript{1}}&0.012\textcolor{gray}{\textsuperscript{7}}&0.018\textcolor{gray}{\textsuperscript{8}}&0.022\textcolor{gray}{\textsuperscript{8}}&\textcolor{red}{0.029}\textcolor{gray}{\textsuperscript{1}}&\textcolor{red}{0.041}\textcolor{gray}{\textsuperscript{1}}\\
SwinIR-classical\citep{swinir}&\textcolor{red}{40.372}\textcolor{gray}{\textsuperscript{1}}&\textcolor{red}{35.125}\textcolor{gray}{\textsuperscript{1}}&\textcolor{red}{32.370}\textcolor{gray}{\textsuperscript{1}}&-&-&32.750\textcolor{gray}{\textsuperscript{5}}&48.380\textcolor{gray}{\textsuperscript{4}}&58.579\textcolor{gray}{\textsuperscript{4}}&-&-&0.012\textcolor{gray}{\textsuperscript{5}}&0.017\textcolor{gray}{\textsuperscript{5}}&0.021\textcolor{gray}{\textsuperscript{5}}&-&-\\
ITSRN\citep{ITSRN}&31.254\textcolor{gray}{\textsuperscript{18}}&26.178\textcolor{gray}{\textsuperscript{18}}&25.876\textcolor{gray}{\textsuperscript{18}}&25.619\textcolor{gray}{\textsuperscript{5}}&21.566\textcolor{gray}{\textsuperscript{5}}&41.928\textcolor{gray}{\textsuperscript{17}}&53.459\textcolor{gray}{\textsuperscript{17}}&62.927\textcolor{gray}{\textsuperscript{17}}&74.260\textcolor{gray}{\textsuperscript{5}}&88.160\textcolor{gray}{\textsuperscript{4}}&0.015\textcolor{gray}{\textsuperscript{17}}&0.020\textcolor{gray}{\textsuperscript{17}}&0.024\textcolor{gray}{\textsuperscript{17}}&0.030\textcolor{gray}{\textsuperscript{5}}&0.042\textcolor{gray}{\textsuperscript{5}}\\
Bicubic&32.790\textcolor{gray}{\textsuperscript{17}}&28.896\textcolor{gray}{\textsuperscript{17}}&26.568\textcolor{gray}{\textsuperscript{17}}&23.450\textcolor{gray}{\textsuperscript{6}}&18.736\textcolor{gray}{\textsuperscript{6}}&50.571\textcolor{gray}{\textsuperscript{18}}&64.949\textcolor{gray}{\textsuperscript{18}}&73.412\textcolor{gray}{\textsuperscript{18}}&82.567\textcolor{gray}{\textsuperscript{6}}&93.012\textcolor{gray}{\textsuperscript{6}}&0.018\textcolor{gray}{\textsuperscript{18}}&0.024\textcolor{gray}{\textsuperscript{18}}&0.029\textcolor{gray}{\textsuperscript{18}}&0.036\textcolor{gray}{\textsuperscript{6}}&0.050\textcolor{gray}{\textsuperscript{6}}\\
HAT-S\citep{hat}&39.784\textcolor{gray}{\textsuperscript{7}}&33.799\textcolor{gray}{\textsuperscript{13}}&31.057\textcolor{gray}{\textsuperscript{14}}&-&-&\textcolor{blue}{32.387}\textcolor{gray}{\textsuperscript{2}}&48.047\textcolor{gray}{\textsuperscript{3}}&58.232\textcolor{gray}{\textsuperscript{3}}&-&-&\textcolor{blue}{0.011}\textcolor{gray}{\textsuperscript{2}}&0.017\textcolor{gray}{\textsuperscript{3}}&0.021\textcolor{gray}{\textsuperscript{3}}&-&-\\
HAT\citep{hat}&39.784\textcolor{gray}{\textsuperscript{7}}&33.913\textcolor{gray}{\textsuperscript{11}}&31.196\textcolor{gray}{\textsuperscript{12}}&-&-&\textcolor{blue}{32.387}\textcolor{gray}{\textsuperscript{2}}&\textcolor{blue}{47.815}\textcolor{gray}{\textsuperscript{2}}&\textcolor{blue}{58.052}\textcolor{gray}{\textsuperscript{2}}&-&-&\textcolor{blue}{0.011}\textcolor{gray}{\textsuperscript{2}}&\textcolor{blue}{0.017}\textcolor{gray}{\textsuperscript{2}}&\textcolor{blue}{0.021}\textcolor{gray}{\textsuperscript{2}}&-&-\\
HDSRNet\citep{hdsr}&39.458\textcolor{gray}{\textsuperscript{12}}&34.203\textcolor{gray}{\textsuperscript{9}}&31.325\textcolor{gray}{\textsuperscript{8}}&-&-&35.245\textcolor{gray}{\textsuperscript{13}}&51.690\textcolor{gray}{\textsuperscript{13}}&61.709\textcolor{gray}{\textsuperscript{13}}&-&-&0.012\textcolor{gray}{\textsuperscript{9}}&0.019\textcolor{gray}{\textsuperscript{12}}&0.023\textcolor{gray}{\textsuperscript{13}}&-&-\\
GRLBase\citep{grl}&39.990\textcolor{gray}{\textsuperscript{4}}&33.840\textcolor{gray}{\textsuperscript{12}}&31.121\textcolor{gray}{\textsuperscript{13}}&-&-&\textcolor{red}{31.807}\textcolor{gray}{\textsuperscript{1}}&\textcolor{red}{47.263}\textcolor{gray}{\textsuperscript{1}}&\textcolor{red}{57.296}\textcolor{gray}{\textsuperscript{1}}&-&-&\textcolor{red}{0.011}\textcolor{gray}{\textsuperscript{1}}&\textcolor{red}{0.017}\textcolor{gray}{\textsuperscript{1}}&\textcolor{red}{0.021}\textcolor{gray}{\textsuperscript{1}}&-&-\\
GRLSmall\citep{grl}&39.544\textcolor{gray}{\textsuperscript{10}}&33.679\textcolor{gray}{\textsuperscript{14}}&31.027\textcolor{gray}{\textsuperscript{15}}&-&-&32.714\textcolor{gray}{\textsuperscript{4}}&48.593\textcolor{gray}{\textsuperscript{5}}&58.805\textcolor{gray}{\textsuperscript{5}}&-&-&0.012\textcolor{gray}{\textsuperscript{4}}&0.017\textcolor{gray}{\textsuperscript{4}}&0.021\textcolor{gray}{\textsuperscript{4}}&-&-\\
GRLTiny\citep{grl}&39.205\textcolor{gray}{\textsuperscript{16}}&33.197\textcolor{gray}{\textsuperscript{15}}&30.556\textcolor{gray}{\textsuperscript{16}}&-&-&33.394\textcolor{gray}{\textsuperscript{6}}&49.550\textcolor{gray}{\textsuperscript{6}}&59.835\textcolor{gray}{\textsuperscript{6}}&-&-&0.012\textcolor{gray}{\textsuperscript{6}}&0.018\textcolor{gray}{\textsuperscript{6}}&0.022\textcolor{gray}{\textsuperscript{6}}&-&-\\

    \bottomrule
    \end{tabular}
}
    
    \caption{Comparison between our proposed FSDS metric and $\ell_1$ norm in both frequency domain and image domain. Items with the highest mean values are highlighted in red and secondary mean values in blue. The gray superscripts denote the ranking order.}
    \label{tab:l1}
\end{table}

\begin{table}[h]
    \centering
\resizebox{\textwidth}{!}{
    \begin{tabular}{lcccccccccccccccc}
    \toprule 
         \multirow{3}*{Method}&\multicolumn{5}{c}{{FSDS (Ours)}}&\multicolumn{5}{c}{{$\ell_2$ Norm in Frequency Domain}}&\multicolumn{5}{c}{{$\ell_2$ Norm}}\\
         \cmidrule(l{2pt}r{2pt}){2-6}\cmidrule(l{2pt}r{2pt}){7-11}\cmidrule(l{2pt}r{2pt}){12-16}&{$\times 2$}&{$\times 3$}&{$\times 4$}&{$\times 6$}&{$\times 12$}&{$\times 2$}&{$\times 3$}&{$\times 4$}&{$\times 6$}&{$\times 12$}&{$\times 2$}&{$\times 3$}&{$\times 4$}&{$\times 6$}&{$\times 12$}\\
         
         \midrule
         EDSR\citep{edsr}&39.210\textcolor{gray}{\textsuperscript{15}}&34.148\textcolor{gray}{\textsuperscript{10}}&31.380\textcolor{gray}{\textsuperscript{7}}&-&-&4966.905\textcolor{gray}{\textsuperscript{13}}&11163.148\textcolor{gray}{\textsuperscript{15}}&17206.009\textcolor{gray}{\textsuperscript{16}}&-&-&4966.906\textcolor{gray}{\textsuperscript{13}}&11163.149\textcolor{gray}{\textsuperscript{15}}&17206.010\textcolor{gray}{\textsuperscript{16}}&-&-\\
EDSR-LIIF\citep{edsr}&39.371\textcolor{gray}{\textsuperscript{13}}&34.535\textcolor{gray}{\textsuperscript{6}}&31.321\textcolor{gray}{\textsuperscript{9}}&28.446\textcolor{gray}{\textsuperscript{4}}&23.147\textcolor{gray}{\textsuperscript{3}}&4967.727\textcolor{gray}{\textsuperscript{14}}&11163.112\textcolor{gray}{\textsuperscript{14}}&16947.235\textcolor{gray}{\textsuperscript{15}}&27007.748\textcolor{gray}{\textsuperscript{4}}&49280.883\textcolor{gray}{\textsuperscript{4}}&4967.727\textcolor{gray}{\textsuperscript{14}}&11163.113\textcolor{gray}{\textsuperscript{14}}&16947.236\textcolor{gray}{\textsuperscript{15}}&27007.749\textcolor{gray}{\textsuperscript{4}}&49280.886\textcolor{gray}{\textsuperscript{4}}\\
EDSR-OPESR\citep{edsr}&39.798\textcolor{gray}{\textsuperscript{6}}&34.630\textcolor{gray}{\textsuperscript{5}}&31.286\textcolor{gray}{\textsuperscript{11}}&-&-&5247.229\textcolor{gray}{\textsuperscript{16}}&11117.694\textcolor{gray}{\textsuperscript{12}}&16818.598\textcolor{gray}{\textsuperscript{14}}&-&-&5247.229\textcolor{gray}{\textsuperscript{16}}&11117.695\textcolor{gray}{\textsuperscript{12}}&16818.599\textcolor{gray}{\textsuperscript{14}}&-&-\\
EDSR-SRNO\citep{edsr}&39.533\textcolor{gray}{\textsuperscript{11}}&34.527\textcolor{gray}{\textsuperscript{7}}&31.448\textcolor{gray}{\textsuperscript{6}}&28.458\textcolor{gray}{\textsuperscript{3}}&22.778\textcolor{gray}{\textsuperscript{4}}&4794.903\textcolor{gray}{\textsuperscript{9}}&10845.390\textcolor{gray}{\textsuperscript{10}}&16453.701\textcolor{gray}{\textsuperscript{10}}&26236.246\textcolor{gray}{\textsuperscript{3}}&48097.246\textcolor{gray}{\textsuperscript{3}}&4794.903\textcolor{gray}{\textsuperscript{9}}&10845.391\textcolor{gray}{\textsuperscript{10}}&16453.702\textcolor{gray}{\textsuperscript{10}}&26236.247\textcolor{gray}{\textsuperscript{3}}&48097.249\textcolor{gray}{\textsuperscript{3}}\\
EDSR-LTE\citep{edsr}&39.290\textcolor{gray}{\textsuperscript{14}}&34.328\textcolor{gray}{\textsuperscript{8}}&31.303\textcolor{gray}{\textsuperscript{10}}&-&-&4907.855\textcolor{gray}{\textsuperscript{11}}&11032.252\textcolor{gray}{\textsuperscript{11}}&16781.621\textcolor{gray}{\textsuperscript{13}}&-&-&4907.855\textcolor{gray}{\textsuperscript{11}}&11032.252\textcolor{gray}{\textsuperscript{11}}&16781.622\textcolor{gray}{\textsuperscript{13}}&-&-\\
RDN\citep{rdn}&40.022\textcolor{gray}{\textsuperscript{3}}&32.946\textcolor{gray}{\textsuperscript{16}}&31.646\textcolor{gray}{\textsuperscript{4}}&-&-&4820.540\textcolor{gray}{\textsuperscript{10}}&12055.745\textcolor{gray}{\textsuperscript{16}}&16462.594\textcolor{gray}{\textsuperscript{11}}&-&-&4820.541\textcolor{gray}{\textsuperscript{10}}&12055.746\textcolor{gray}{\textsuperscript{16}}&16462.595\textcolor{gray}{\textsuperscript{11}}&-&-\\
RDN-LIIF\citep{rdn}&39.690\textcolor{gray}{\textsuperscript{9}}&34.831\textcolor{gray}{\textsuperscript{3}}&31.832\textcolor{gray}{\textsuperscript{3}}&\textcolor{red}{28.894}\textcolor{gray}{\textsuperscript{1}}&\textcolor{red}{23.778}\textcolor{gray}{\textsuperscript{1}}&4627.806\textcolor{gray}{\textsuperscript{8}}&10477.303\textcolor{gray}{\textsuperscript{8}}&15993.797\textcolor{gray}{\textsuperscript{8}}&\textcolor{blue}{25720.766}\textcolor{gray}{\textsuperscript{2}}&\textcolor{blue}{47650.681}\textcolor{gray}{\textsuperscript{2}}&4627.806\textcolor{gray}{\textsuperscript{8}}&10477.304\textcolor{gray}{\textsuperscript{8}}&15993.798\textcolor{gray}{\textsuperscript{8}}&\textcolor{blue}{25720.768}\textcolor{gray}{\textsuperscript{2}}&\textcolor{blue}{47650.683}\textcolor{gray}{\textsuperscript{2}}\\
RDN-OPESR\citep{rdn}&\textcolor{blue}{40.188}\textcolor{gray}{\textsuperscript{2}}&\textcolor{blue}{34.959}\textcolor{gray}{\textsuperscript{2}}&31.475\textcolor{gray}{\textsuperscript{5}}&-&-&5061.535\textcolor{gray}{\textsuperscript{15}}&10609.168\textcolor{gray}{\textsuperscript{9}}&16034.925\textcolor{gray}{\textsuperscript{9}}&-&-&5061.535\textcolor{gray}{\textsuperscript{15}}&10609.169\textcolor{gray}{\textsuperscript{9}}&16034.926\textcolor{gray}{\textsuperscript{9}}&-&-\\
RDN-LTE\citep{rdn}&39.825\textcolor{gray}{\textsuperscript{5}}&34.749\textcolor{gray}{\textsuperscript{4}}&\textcolor{blue}{31.837}\textcolor{gray}{\textsuperscript{2}}&\textcolor{blue}{28.654}\textcolor{gray}{\textsuperscript{2}}&\textcolor{blue}{23.346}\textcolor{gray}{\textsuperscript{2}}&4576.652\textcolor{gray}{\textsuperscript{7}}&10349.273\textcolor{gray}{\textsuperscript{7}}&15808.890\textcolor{gray}{\textsuperscript{7}}&\textcolor{red}{25384.597}\textcolor{gray}{\textsuperscript{1}}&\textcolor{red}{46955.729}\textcolor{gray}{\textsuperscript{1}}&4576.652\textcolor{gray}{\textsuperscript{7}}&10349.274\textcolor{gray}{\textsuperscript{7}}&15808.891\textcolor{gray}{\textsuperscript{7}}&\textcolor{red}{25384.598}\textcolor{gray}{\textsuperscript{1}}&\textcolor{red}{46955.731}\textcolor{gray}{\textsuperscript{1}}\\
SwinIR-classical\citep{swinir}&\textcolor{red}{40.372}\textcolor{gray}{\textsuperscript{1}}&\textcolor{red}{35.125}\textcolor{gray}{\textsuperscript{1}}&\textcolor{red}{32.370}\textcolor{gray}{\textsuperscript{1}}&-&-&4222.226\textcolor{gray}{\textsuperscript{5}}&9617.212\textcolor{gray}{\textsuperscript{5}}&14804.533\textcolor{gray}{\textsuperscript{5}}&-&-&4222.226\textcolor{gray}{\textsuperscript{5}}&9617.212\textcolor{gray}{\textsuperscript{5}}&14804.534\textcolor{gray}{\textsuperscript{5}}&-&-\\
ITSRN\citep{ITSRN}&31.254\textcolor{gray}{\textsuperscript{18}}&26.178\textcolor{gray}{\textsuperscript{18}}&25.876\textcolor{gray}{\textsuperscript{18}}&25.619\textcolor{gray}{\textsuperscript{5}}&21.566\textcolor{gray}{\textsuperscript{5}}&7483.582\textcolor{gray}{\textsuperscript{17}}&12244.897\textcolor{gray}{\textsuperscript{17}}&17878.399\textcolor{gray}{\textsuperscript{17}}&27664.054\textcolor{gray}{\textsuperscript{5}}&49417.061\textcolor{gray}{\textsuperscript{5}}&7483.582\textcolor{gray}{\textsuperscript{17}}&12244.898\textcolor{gray}{\textsuperscript{17}}&17878.400\textcolor{gray}{\textsuperscript{17}}&27664.056\textcolor{gray}{\textsuperscript{5}}&49417.064\textcolor{gray}{\textsuperscript{5}}\\
Bicubic&32.790\textcolor{gray}{\textsuperscript{17}}&28.896\textcolor{gray}{\textsuperscript{17}}&26.568\textcolor{gray}{\textsuperscript{17}}&23.450\textcolor{gray}{\textsuperscript{6}}&18.736\textcolor{gray}{\textsuperscript{6}}&10873.900\textcolor{gray}{\textsuperscript{18}}&19564.473\textcolor{gray}{\textsuperscript{18}}&26817.580\textcolor{gray}{\textsuperscript{18}}&38541.675\textcolor{gray}{\textsuperscript{6}}&64011.698\textcolor{gray}{\textsuperscript{6}}&10873.901\textcolor{gray}{\textsuperscript{18}}&19564.474\textcolor{gray}{\textsuperscript{18}}&26817.581\textcolor{gray}{\textsuperscript{18}}&38541.677\textcolor{gray}{\textsuperscript{6}}&64011.702\textcolor{gray}{\textsuperscript{6}}\\
HAT-S\citep{hat}&39.784\textcolor{gray}{\textsuperscript{7}}&33.799\textcolor{gray}{\textsuperscript{13}}&31.057\textcolor{gray}{\textsuperscript{14}}&-&-&\textcolor{blue}{4107.382}\textcolor{gray}{\textsuperscript{2}}&9423.447\textcolor{gray}{\textsuperscript{3}}&14498.509\textcolor{gray}{\textsuperscript{3}}&-&-&\textcolor{blue}{4107.383}\textcolor{gray}{\textsuperscript{2}}&9423.448\textcolor{gray}{\textsuperscript{3}}&14498.510\textcolor{gray}{\textsuperscript{3}}&-&-\\
HAT\citep{hat}&39.784\textcolor{gray}{\textsuperscript{7}}&33.913\textcolor{gray}{\textsuperscript{11}}&31.196\textcolor{gray}{\textsuperscript{12}}&-&-&\textcolor{blue}{4107.382}\textcolor{gray}{\textsuperscript{2}}&\textcolor{blue}{9326.613}\textcolor{gray}{\textsuperscript{2}}&\textcolor{blue}{14407.457}\textcolor{gray}{\textsuperscript{2}}&-&-&\textcolor{blue}{4107.383}\textcolor{gray}{\textsuperscript{2}}&\textcolor{blue}{9326.613}\textcolor{gray}{\textsuperscript{2}}&\textcolor{blue}{14407.458}\textcolor{gray}{\textsuperscript{2}}&-&-\\
HDSRNet\citep{hdsr}&39.458\textcolor{gray}{\textsuperscript{12}}&34.203\textcolor{gray}{\textsuperscript{9}}&31.325\textcolor{gray}{\textsuperscript{8}}&-&-&4948.237\textcolor{gray}{\textsuperscript{12}}&11121.912\textcolor{gray}{\textsuperscript{13}}&16697.185\textcolor{gray}{\textsuperscript{12}}&-&-&4948.238\textcolor{gray}{\textsuperscript{12}}&11121.912\textcolor{gray}{\textsuperscript{13}}&16697.186\textcolor{gray}{\textsuperscript{12}}&-&-\\
GRLBase\citep{grl}&39.990\textcolor{gray}{\textsuperscript{4}}&33.840\textcolor{gray}{\textsuperscript{12}}&31.121\textcolor{gray}{\textsuperscript{13}}&-&-&\textcolor{red}{3937.013}\textcolor{gray}{\textsuperscript{1}}&\textcolor{red}{9073.913}\textcolor{gray}{\textsuperscript{1}}&\textcolor{red}{13959.078}\textcolor{gray}{\textsuperscript{1}}&-&-&\textcolor{red}{3937.013}\textcolor{gray}{\textsuperscript{1}}&\textcolor{red}{9073.913}\textcolor{gray}{\textsuperscript{1}}&\textcolor{red}{13959.079}\textcolor{gray}{\textsuperscript{1}}&-&-\\
GRLSmall\citep{grl}&39.544\textcolor{gray}{\textsuperscript{10}}&33.679\textcolor{gray}{\textsuperscript{14}}&31.027\textcolor{gray}{\textsuperscript{15}}&-&-&4170.853\textcolor{gray}{\textsuperscript{4}}&9585.326\textcolor{gray}{\textsuperscript{4}}&14742.787\textcolor{gray}{\textsuperscript{4}}&-&-&4170.853\textcolor{gray}{\textsuperscript{4}}&9585.327\textcolor{gray}{\textsuperscript{4}}&14742.788\textcolor{gray}{\textsuperscript{4}}&-&-\\
GRLTiny\citep{grl}&39.205\textcolor{gray}{\textsuperscript{16}}&33.197\textcolor{gray}{\textsuperscript{15}}&30.556\textcolor{gray}{\textsuperscript{16}}&-&-&4384.364\textcolor{gray}{\textsuperscript{6}}&10079.017\textcolor{gray}{\textsuperscript{6}}&15488.316\textcolor{gray}{\textsuperscript{6}}&-&-&4384.364\textcolor{gray}{\textsuperscript{6}}&10079.018\textcolor{gray}{\textsuperscript{6}}&15488.317\textcolor{gray}{\textsuperscript{6}}&-&-\\

    \bottomrule
    \end{tabular}
    }
    
    \caption{Comparison between our proposed FSDS metric and $\ell_2$ norm in both frequency domain and image domain. Items with the highest mean values are highlighted in red and secondary mean values in blue. The gray superscripts denote the ranking order. To demonstrate Parseval's theorem, we omit the mean operation when calculating $\ell_2$ norm.}
    \label{tab:l2}
\end{table}
\section{How G(I) is Trained and How G(I) and H(I) Vary Dependently}
\begin{figure}
    \centering
    \includegraphics[scale=0.2]{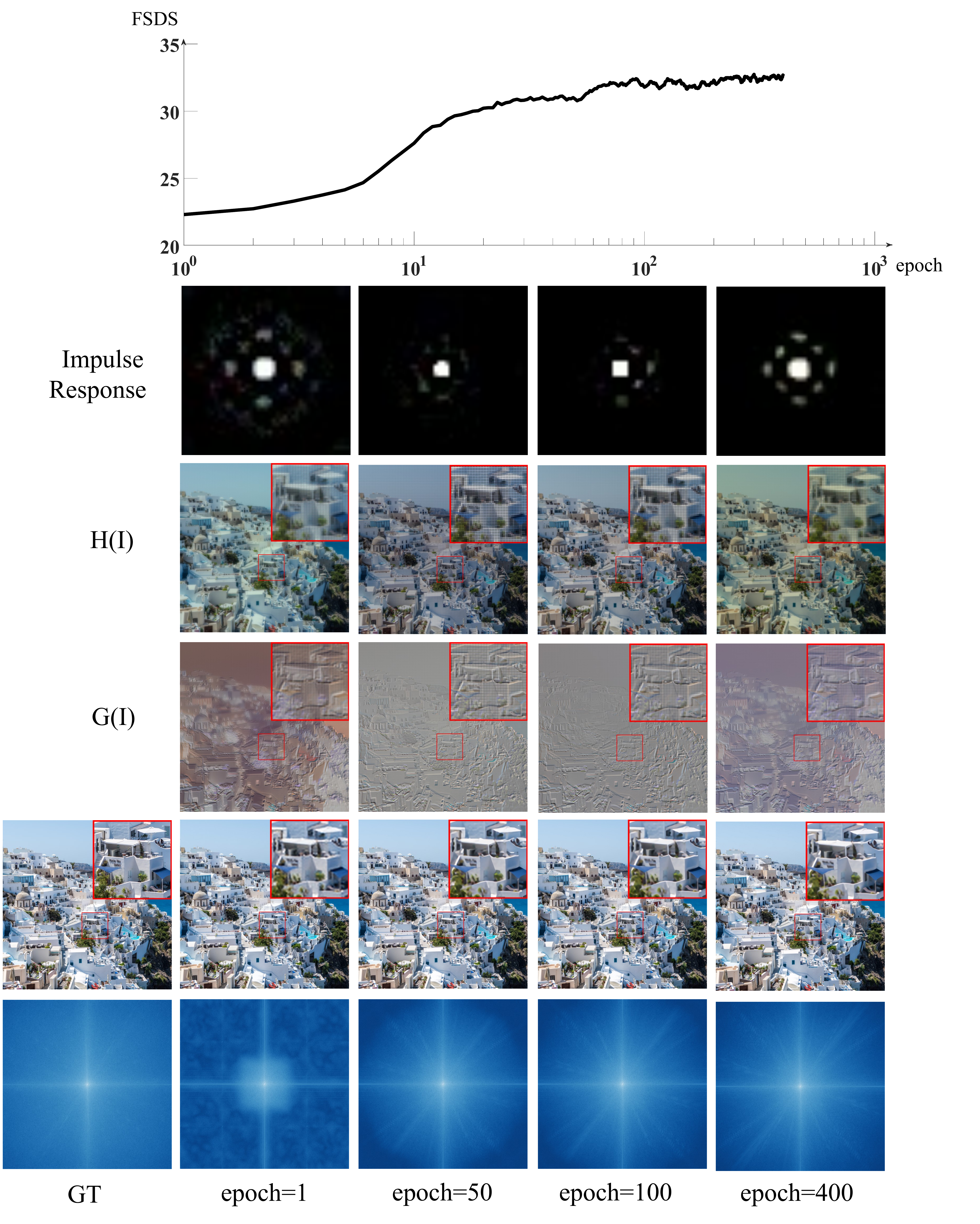}
    \caption{$N(I)$, $H(I)$ and $G(I)$ from different epoches during training.}
    \label{fig:during_training}
\end{figure}
 To better explain how $G(I)$ is trained and how $G(I)$ and $H(I)$ vary dependently, we conduct a new experiment on observing their varying progress during training. We train a vanilla RDN \citep{rdn} network from scratch and obtain the impulse response, $H(I)$, $G(I)$, and $N(I)$ from each epoch (see \Cref{fig:during_training}). As shown in the figure, in the second row, the $sinc$ phenomenon becomes clearer along with the training process, this indicates that the network is gradually learning the low-pass filter. The same conclusion can be further supported by observing the variation in $H(I)$ as shown in the third row. In this row, $H(I)$ illustrates the phenomenon of the grid-like distortion vanishing while low-frequency areas getting smoother. The fourth row depict the plot of $G(I)$. This row demonstrates that $G(I)$ is capturing more and more high-frequency information, such as edges. This is very significant especially when comparing the result from Epoch 1 and Epoch 50. Such increase in high-frequncy information can also be found in the frequency spectrum. During the entire training process, the network fails to recover the words on the wall (please see the magnified area, there are words on the wall next to the blue awning). Observing $G(I)$, we can find no sign of words as well. This indicates that the network treats it as low-frequency information and ignores it, pointing the way for future network improvements.
\section{Code Repositories}
\begin{table}[H]
    \centering
    \begin{tabular}{lp{8cm}lll}
    \toprule
         Abbreviate& Title& Publication&Year&Code Link  \\
         \midrule
         EDSR \citep{edsr}&Enhanced Deep Residual Networks for Single Image Super-Resolution& CVPRW& 2017& \href{https://github.com/sanghyun-son/EDSR-PyTorchGithub}{Github}\\
         LIIF\cite{liif}&Learning Continuous Image Representation with Local Implicit Image Function&CVPR&2021&\href{https://github.com/yinboc/liif}{Github}\\
         OPE-SR\cite{OPE-SR}&OPE-SR: Orthogonal Position Encoding for Designing a Parameter-Free Upsampling Module in Arbitrary-Scale Image Super-Resolution&CVPR&2023&\href{https://github.com/gaochao-s/ope-sr}{Github}\\
         SRNO\cite{SRNO}&Super-Resolution Neural Operator&CVPR&2023&\href{https://github.com/2y7c3/Super-Resolution-Neural-Operator}{Github}\\
         LTE \citep{LTE}&Local Texture Estimator for Implicit Representation Function&CVPR&2022&\href{https://github.com/jaewon-lee-b/lte}{Github}\\
         RDN \citep{rdn}&Residual Dense Network for Image Super-Resolution &CVPR&2018&\href{https://github.com/yulunzhang/RDN}{Github}\\
         SwinIR \citep{swinir}&SwinIR: Image Restoration Using Swin Transformer&ICCV&2021&\href{https://github.com/JingyunLiang/SwinIR}{Github}\\
         ITSRN \citep{ITSRN}&Implicit Transformer Network for Screen Content Image Continuous Super-Resolution&NeurIPS&2021&\href{https://github.com/codyshen0000/ITSRN}{Github}\\
         RCAN \citep{RCAN}&Image Super-Resolution Using Very Deep Residual Channel Attention Networks&ECCV&2018&\href{https://github.com/yulunzhang/RCAN}{Github}\\
         HAT \citep{hat}&Activating More Pixels in Image Super-Resolution Transformer&CVPR&2023&\href{https://github.com/XPixelGroup/HAT}{Github}\\
         HDSRNet \citep{hdsr}&Heterogeneous Dynamic Convolutional Network in Image Super-Resolution&Arxiv&2024&\href{https://github.com/hellloxiaotian/HDSRNet}{Github}\\
         GRL \citep{grl}& Efficient and Explicit Modelling of Image Hierarchies for Image Restoration&CVPR&2023&\href{https://github.com/ofsoundof/GRL-Image-Restoration}{Github}\\
         \bottomrule
    \end{tabular}
    \caption{The papers and repository links used in this paper.}
    \label{tab:code_link}
\end{table}

\end{document}